\documentclass[11pt]{article}
\pdfoutput=1
\usepackage[T1]{fontenc} 
\usepackage{url}
\usepackage[left=4cm, right=4cm]{geometry}
\usepackage{times}
\usepackage{graphicx} 
\usepackage{subfigure}
\usepackage{natbib}
\usepackage[ruled]{algorithm2e}
\setcitestyle{square}
\usepackage{amsmath,amssymb}
\usepackage{multirow}
\usepackage{math_definitions}
\usepackage{authblk}

\title{Distributed Multitask Learning}

\author{
Jialei Wang \\
Department of Computer Science\\
University of Chicago\\
\texttt{jialei@uchicago.edu} 
\and
Mladen Kolar \\
Booth School of Business  \\
University of Chicago \\
\texttt{mkolar@chicagobooth.edu} 
\and
Nathan Srebro \\
Toyota Technological Institute at Chicago \\
\texttt{nati@ttic.edu} 
}

\newcommand{\removed}[1]{}

\begin{document}

\maketitle

\begin{abstract}
  We consider the problem of distributed multi-task learning, where
  each machine learns a separate, but related, task.  Specifically,
  each machine learns a linear predictor in high-dimensional space,
  where all tasks share the same small support.  We present a
  communication-efficient estimator based on the debiased lasso and
  show that it is comparable with the optimal centralized method.
\end{abstract}

\section{Introduction}

Learning multiple tasks simultaneously allows transferring information
between related tasks and for improved performance compared to
learning each tasks separately \citep{caruana1997multitask}.
It has been successfully exploited in, e.g., spam filtering
\citep{weinberger2009feature}, web search \citep{chapelle2010multi},
disease prediction \citep{zhou2013modeling} and eQTL mapping
\citep{kim2010tree}.

Tasks could be related to each other in a number of ways.  In this
paper, we focus on the high-dimensional multi-task setting with joint
support where a few variables are related to all tasks, while others
are not predictive
\citep{tOPTUrlach2005simultaneous,obozinski10support,Lounici2011Oracle}.
The standard approach is to use the mixed $\ell_1/\ell_2$ or
$\ell_1/\ell_\infty$ penalty, as such penalties encourage selection of
variables that affect all tasks. Using a mixed norm penalty leads to
better performance in terms of prediction, estimation and model
selection compared to using the $\ell_1$ norm penalty, which is
equivalent to considering each task separately.

Shared support multi-task learning is generally considered in a
centralized setting where data from all tasks is available on a single
machine, and the estimator is computed using a standard single-thread
algorithm.  With the growth of modern massive data sets, there is a
need to revisit multi-task learning in a distributed setting, where
tasks and data are distributed across machines and communication is
expensive.  In particular, we consider a setting where each machine
holds one ``task'' and its related data.

We develop an efficient distributed algorithm for multi-task learning
that exploits shared sparsity between tasks. Our algorithm (DSML)
requires only one round of communication between the workers and the
central node, involving each machine sending a vector to the central
node and receiving back a support set.  Despite the limited
communication, our algorithm enjoys the same theoretical guarantees,
in terms of the leading term in reasonable regimes and mild
conditions, as the centralized approach.
Table \ref{tab:comparison_sparsistency} summarizes our support
recovery guarantees compared to the centralized (group lasso) and local
(lasso) approaches, while Table \ref{tab:comparison_estimation}
compares the parameter and prediction error guarantees.

\begin{table}[htpb]
\begin{center}
\begin{footnotesize}
\newcommand{\tabincell}[2]{\begin{tabular}{@{}#1@{}}#2\end{tabular}}
\begin{tabular}{c|c|c|c|c}\hline
Approach & Communication & Assumptions  & Min signal strength & Strength type \\\hline
Lasso &0 &\tabincell{c}{Mutual Incoherence \\Sparse Eigenvalue}  & $\sqrt{\frac{\log p}{n}}$ & Element-wise\\\hline
Group lasso &$\Ocal(np)$  & \tabincell{c}{Mutual Incoherence \\Sparse Eigenvalue}  &$\sqrt{\frac{1}{n}\rbr{1 + \frac{\log p}{m}}}$ &Row-wise \\\hline 
DSML & $\Ocal(p)$  &\tabincell{c}{Generalized Coherence \\ Restricted Eigenvalue}   &$\sqrt{\frac{1}{n}\rbr{1 + \frac{\log p}{m}}} + \frac{|S|\log p}{n}$ &Row-wise \\\hline
\end{tabular}
\small \caption{\small Lower bound on coefficients required to
  ensure support recovery with $p$ variables, $m$ tasks, $n$
  samples per task and a true support of size $|S|$.}
\label{tab:comparison_sparsistency}
\end{footnotesize}
\end{center}
\end{table}

\begin{table}[htpb]
\begin{footnotesize}
\newcommand{\tabincell}[2]{\begin{tabular}{@{}#1@{}}#2\end{tabular}}
\begin{center}
\begin{tabular}{c|c|c|c}\hline
Approach  & Assumptions & $\ell_1/\ell_2$ estimation error & Prediction error \\\hline
Lasso  &Restricted Eigenvalue & $ \sqrt{\frac{|S|^2 \log p}{n}}$ & $\frac{|S| \log p}{n}$ \\\hline
Group lasso   & Restricted Eigenvalue & $\frac{|S|}{\sqrt{n}} \sqrt{1 + \frac{\log p}{m}}$ &$\frac{|S|}{n} \rbr{1 + \frac{\log p}{m}}$ \\\hline 
DSML  &\tabincell{c}{Generalized Coherence \\ Restricted Eigenvalue}  & $\frac{|S|}{\sqrt{n}} \sqrt{1 + \frac{\log p}{m}} + \frac{|S|^2 \log p}{n}$ &$\frac{|S|}{n} \rbr{1 + \frac{\log p}{m}} + \frac{|S|^3 (\log p)^2}{n^2}$ \\\hline
\end{tabular}
\end{center}
\small \caption{\small Comparison of parameter estimation errors and
  prediction errors. The DSML guarantees improve over Lasso and have
  the same leading term as the Group lasso as long as $m<n/(|S|^2\log
  p)$.}
\label{tab:comparison_estimation}
\end{footnotesize}
\end{table}

\section{Distributed Learning and Optimization}
\label{section:distributed_learning}

With the increase in the volume of data used for machine learning, and
the availability of distributed computing resources, distributed
learning and the use of distributed optimization for machine learning
has received much attention.

Most work on distributed optimization focuses on ``consensus
problems'', where each machine holds a {\em different} objective
$f_i(\beta)$ and the goal is to communicate between the machines so as
to jointly optimize the average objective $1/m \sum_i f_i(\beta)$,
that is,~to find a single vector $\beta$ that is good for all local
objectives \citep{boyd2011distributed}.  The difficulty of consensus
problems is that the local objectives might be rather different, and,
as a result, one can obtain lower bounds on the amount of
communication that must be exchanged in order to reach a joint
optimum.  In particular, the problem becomes harder as more machines
are involved.

The consensus problem has also been studied in the stochastic setting
\citep{ram2010distributed}, in which each machine receive stochastic
estimates of its local objective.  Thinking of each local objective as
a generalization error w.r.t.~a local distribution, we obtain the
following distributed learning formulation
\citep{DBLP:journals/jmlr/BalcanBFM12}: each machine holds a different
source distribution $\Dcal_i$ from which it can sample, and this
distribution corresponds to a different local generalization error
$f_i=\EE_{(X,y)\sim\Dcal_i}[\textit{loss}(\beta,X,y)]$.  The goal is
to find a single predictor $\beta$ that minimizes the average
generalization error, based on samples sampled at the local nodes.
Again, the problem becomes harder when more machines are involved and
one can obtain lower bounds on the amount of communication
required---\citep{DBLP:journals/jmlr/BalcanBFM12} carry out such an
analysis for several hypothesis classes.

A more typical situation in machine learning is one in which there is
only a single source distribution $\Dcal$, and data from this single
source is distributed randomly across the machines (or equivalently,
each machine has access to the same source distribution
$\Dcal_i=\Dcal$).  Such a problem can be reduced to a consensus
problem by performing consensus optimization of the empirical errors
at each machine. However, such an approach ignores several issues:
first, the local empirical objectives are not arbitrarily different,
but rather quite similar, which can and should be taken advantage of
in optimization \citep{DANE}.  Second, since each machine has access to
the source distribution, there is no lower bound on communication---an
entirely ``local'' approach is possible, were each machine completely
ignores other machines and just uses its own data.  In fact,
increasing the number of machines only makes the problem easier (in
that it can reduce the runtime or number of samples per machine
required to achieve target performance), as additional machines can
always be ignored.  In such a setting, the other relevant baseline is
the ``centralized'' approach, where all data is communicated to a
central machine which computes a predictor centrally.  The goal here
is then to obtain performance close to that of the ``centralized''
approach (and much better than the ``local'' approach), using roughly
the same number of samples, but with low communication and computation
costs.  Such single-source distributed problems have been studied both
in terms of predictive performance
\citep{shamir2014distributed,DBLP:conf/nips/JaggiSTTKHJ14} and
parameter estimation
\citep{DBLP:journals/jmlr/ZhangDW13,DBLP:conf/nips/ZhangDJW13,lee2015communication}.

In this paper we suggest a novel setting that combines aspects of the
above two extremes.  On one hand, we assume that each machine has a
different source distributions $\Dcal_i(X,y)$, corresponding to a
different task, as in consensus problems and in
\citep{DBLP:journals/jmlr/BalcanBFM12}.  For example, each machine
serves a different geographical location, or each is at a different
hospital or school with different characteristics.  But if indeed
there are differences between the source distributions, it is natural
to learn different predictors $\beta_i$ for each machine, so that
$\beta_i$ is good for the distribution typical to that machine.  In
this regard, our distributed multi-task learning problem is more
similar to single-source problems, in that machines could potentially
learn on their own given enough samples and enough time. Furthermore,
availability of other machines just makes the problem easier by
allowing transfer between the machine, thus reducing the sample
complexity and runtime.  The goal, then, is to leverage as much
transfer as possible, while limiting communication and runtime.  As
with single-source problems, we compare our method to the two
baselines, where we would like to be much better than the ``local''
approach, achieving performance nearly as good as the ``centralized''
approach, but with minimal communication and efficient runtime.

To the best of our knowledge, the only previous discussion of
distributed multi-task learning is \citep{dinuzzo2011client}, which
considered a different setting with an almost orthogonal goal: a
client-server architecture, where the server collects data from
different clients, and send sufficient information that might be
helpful for each client to solve its own task. Their emphasis is on
preserving privacy, but their architecture is communication-heavy as
the entire data set is communicated to the central server, as in the
``centralized'' bases line.  On the other hand, we are mostly
concerned with communication costs, but, for the time being, do not
address privacy concerns.

\section{Preliminaries}

We consider the following multi-task linear regression model with $m$ tasks:
\begin{equation}
  \label{eq:model}
  y_t = X_t \beta^*_t + \epsilon_t,\qquad
  t=1,\ldots,m,
\end{equation}
where $X_t \in \RR^{n_t \times p}$, $y_t \in \RR^{n_t}$, and
$\epsilon_t \sim N(0, \sigma_t^2 I) \in \RR^{n_t}$ is a noise vector,
and $ \beta^*_t$ is the unknown vector of coefficients for the task
$t$. For notation simplicity we assume each task has equal sample size
and the same noise level, that is, we assume, $n_1 = n_2 = \ldots = n$
and $\sigma_1 = \sigma_2 = \ldots = \sigma$.  We will be working in a
high-dimensional regime with $p$ possibly larger than $n$, however, we
will assume that each $\beta^*_t$ is sparse, that is, few components
of $\beta^*_t$ are different from zero. Furthermore, we assume that
the support between the tasks is shared. In particular, let $S_t =
{\rm support}(\beta^*_t) = \{j \in [p] : \beta_{tj} \neq 0 \}$, with
$S_1 = S_2 = \ldots = S$ and $s = |S| \ll n$. Suppose the data sets
$(X_1,y_1),\ldots,(X_m,y_m)$ are distributed across machines, our goal
is to estimate $\{\beta_t^*\}_{t=1}^m$ as accurately as possible,
while maintaining low communication cost.

The lasso estimate for each task $t$ is given by:
\begin{align}
\label{eqn:local_lasso}
\hat \beta_t = 
\arg \min_{\beta_t} \frac{1}{n} 
\norm{ y_t - X_t \beta_t }_2^2 + \lambda_t \norm{\beta_t}_1.
\end{align}
The multi-task estimates are given by the joint optimization:
\begin{align}
\label{eq:group_lasso}
\{ \hat \beta_t \}_{t=1}^m = 
\arg \min_{ \{\beta_t\}_{t=1}^m } 
\frac{1}{mn} \sum_{t=1} \norm{y_t - X_t \beta_t}_2^2 
+ \lambda \pen(\{\beta_t\}_{t=1}^m),
\end{align}  
where $\pen(\{\beta_t\}_{t=1}^m)$ is the regularizaton that promote
group sparse solutions. For example, the group lasso penalty uses
$\pen(\{\beta_t\}_{t=1}^m) = \sum_{j \in [p]} \sqrt{\sum_{t \in m}
  \beta_{tj}^2}$ \citep{Yuan2006Model}, while the iCAP uses
$\pen(\{\beta_t\}_{t=1}^m) = \sum_{j \in [p]}
\max_{t=1,\ldots,m} |\beta_{tj}| $
\citep{zhao2009icap}\removed{han09blockwise}.  
In a distributed setting, one could potentially minimize
\eqref{eq:group_lasso} using a distributed consensus procedure (see
Section \ref{section:distributed_learning}), but such an approach
would generally require multiple round of communication.  Our
procedure, described in the next section, lies in between the local
lasso \eqref{eqn:local_lasso} and centralized estimate
\eqref{eq:group_lasso}, requiring only one round of communication to
compute, while still ensuring much of the statistical benefits of
using group regularization.

\section{Methodology}

In this section, we detail our procedure for performing estimation
under model in \eqref{eq:model}. Algorithm \ref{alg:dsml} provides an
outline of the steps executed by the worker nodes and the master node,
which are explained in details below.

\begin{algorithm}[hpt]
\SetAlgoLined
\underline{\textbf{Workers:}}\\
\For{$t=1, 2, \ldots, m$}{
Each worker obtains $\hat \beta_t$ as a solution to a local lasso in \eqref{eqn:local_lasso}\;
Each worker obtains $\hat \beta_t^u$ the debiased lasso estimate  in
\eqref{eqn:debiasing} and sends it to the master\;
\If{Receive $\hat S(\Lambda)$ from the master}{
Calculate final estimate $\tilde \beta_t$ in \eqref{eq:local_thr}.
}
}
\underline{\textbf{Master:}}\\
\If{Receive $\{\hat \beta^u_t\}_{t=1}^m$ from all workers}{
Compute $\hat S(\Lambda)$ by group hard thresholding
in\eqref{eq:hard_thr} and send the result back to every worker.
}
\caption{DSML:Distributed debiased Sparse Multi-task Lasso.}
\label{alg:dsml}
\end{algorithm}

Recall that each worker node contains data for one task. That is, a
node $t$ contains data $(X_t, y_t)$.  In the first step, each worker
node solves a lasso problem locally, that is, a node $t$ minimizes the
program in \eqref{eqn:local_lasso} and obtains $\hat
\beta_t$. Next, a worker node constructs a debiased lasso
estimator $\hat \beta_t^u$ by performing one Newton step update on the
loss function, starting at the estimated value $\hat \beta_t$:
\begin{align}
\label{eqn:debiasing}
\hat \beta_t^u = 
\hat \beta_t + n^{-1} M_t X_t^T (y_t - X_t \hat \beta_t ),
\end{align}
where $n^{-1}X_t^T (y_t - X_t \hat \beta_t )$ is a subgradient of the
$\ell_1$ norm and the matrix $M_t \in \RR^{p \times p}$ serves as an
approximate inverse of the Hessian.  The idea of debiasing the lasso
estimator was introduced in the recent literature on statistical
inference in high-dimensions \citep{Zhang2011Confidence,
  Geer2013asymptotically, Javanmard2013Confidence}.  By removing the
bias introduced through the $\ell_1$ penalty, one can estimate the
sampling distribution of a component of $\hat \beta_t^u$ and make
inference about the unknown parameter of interest.  In our paper, we
will also utilize the sampling distribution of the debiased estimator,
however, with a different goal in mind. The above mentioned papers
proposed different techniques to construct the matrix $M$.  Here, we
adopt the approach proposed in \citep{Javanmard2013Confidence}, as it
leads to weakest assumption on the model in \eqref{eq:model}: each
machine uses a matrix $M_t = (\hat m_{tj})_{j=1}^{p}$ with rows:
\begin{align*}
\hat m_{tj} = \arg\min_{m_j \in \RR^p}\quad  m_j^T \hat\Sigma_t m_j 
\quad\text{subject to}\quad
\norm{\hat\Sigma_t m_j - e_j }_{\infty} \leq \mu. 
\end{align*}
where $e_j$ is the vector with $j$-th component equal to 1 and 0
otherwise and $\hat \Sigma_t = n^{-1} X_t^T X_t $.

After each worker obtains the debiased estimator $\hat \beta^u_t$, it
sends it to the central machine. After debiasing, the estimator is no
longer sparse and as a result each worker communicates $p$ numbers to
the master node. It is at the master where shared sparsity between the
task coefficients gets utilized. The master node concatenates the
received estimators into a matrix $\hat B = (\hat \beta^u_1, \hat
\beta^u_2,...,\hat \beta^u_m)$. Let $\hat B_j$ be the $j$-th row of
$\hat B$. The master performs the hard group thresholding to obtain an
estimate of $S$ as
\begin{align}
  \label{eq:hard_thr}
  \hat S(\Lambda) = \{j \mid \norm{\hat B_j}_2 > \Lambda \}.
\end{align}
The estimated support $\hat
S(\Lambda)$ is communicated back to each worker, which then use the
estimate of the support to filter their local estimate. In particular,
each worker produces the final estimate:
\begin{align}
\label{eq:local_thr}
\tilde\beta_{tj} = \left\{
\begin{array}{cl}
\hat \beta_{tj}^u & \text{if } j \in \hat S(\Lambda)\\
0 & \text{otherwise.}
\end{array}
\right.
\end{align}

\paragraph{Extension to multitask classification.}

DSML can be generalized to estimate multi-task generalized linear
models. We be briefly outline how to extend DSML to a multi-task
logistic regression model, where $y_{tk} \in \{-1,1\}$ and:
\begin{align}
\label{eq:logistic_regression_model}
P(y_{tk}|X_{tk}) = \frac{\exp\rbr{\frac{1}{2} y_{tk} X_{tk} \beta^*_t}}{\exp\rbr{-\frac{1}{2} y_{tk} X_{tk} \beta^*_t} + \exp\rbr{\frac{1}{2} y_{tk} X_{tk} \beta^*_t}}, \quad \forall k = 1,\ldots,n, t = 1,\ldots,m.
\end{align}
First, each worker solves the $\ell_1$-regularized
logistic regression problem
\[
\hat \beta_t = 
\arg \min_{\beta_t} 
\frac{1}{n} \sum_{k \in [n]} \log(1+\exp(- y_{tk} X_{tk} \beta_t)) 
+ \lambda_t \norm{\beta_t}_1.
\]
Let $W_t \in \RR^{n \times n}$ be a diagonal weighting matrix, with a
$k$-th diagonal element 
\[
W_{t(kk)} = 
\frac{1}{1+\exp(- X_{tk} \hat \beta_t)} 
\cdot \frac{\exp(- X_{tk} \hat \beta_t)}{1+\exp(-X_{tk} \hat \beta_t)},
\]
which will be used to approximately invert the Hessian matrix of the
logistic loss. The matrix $M_t = (\hat m_{tj})_{j=1}^p$, which serves
as an approximate inverse of the Hessian, in the case of logistic
regression can be obtained as a solution to the following optimization
problem:
\begin{align*}
\hat m_{tj} = \arg\min_{m_{tj} \in \RR^p}
\quad m_{tj}^T  {X_t}^T W_t X_t m_{tj} 
\quad\text{subject to}\quad 
\norm{n^{-1} X_t^T W_t X_t m_{tj} - e_j }_{\infty} \leq \mu. 
\end{align*}
Finally, the debiased estimator is obtained as
\[
\hat \beta_t^u = \hat \beta_t + 
n^{-1}  M_t {X_t}^T 
\left(\frac{1}{2}(y_{t}+1) - \rbr{1+\exp(-X_{t} \hat \beta_t)}^{-1} \right),
\]
and then communicated to the master node. The rest of procedure is as
described before.

\section{Theoretical Analysis}

In this section, we present our main theoretical results for the
DSML procedure described in the previous section. We start by describing
assumptions that we make on the model in \eqref{eq:model}. Our results are based on the random design analysis, and we also discuss fixed design case in appendix. Let the data $X_t$ for $t$-task are drawn from a subgaussian random vector with covariance matrix $\EE[n^{-1} X_t^T X_t] = \Sigma_t$. We assume the subguassian random vectors for every task have bounded subgaussian norm: $\max_{t} \max_{k} \norm{X_{tk}}_{\psi_2} \leq \sigma_X$ \citep{Vershynin2012Introduction}. Let $\lambda_{\min}(\Sigma)$ to be the minimal eigenvalue of $\Sigma$, and $\lambda_{\max}(\Sigma)$ be its maximal eigenvalue. Let $\lambda_{\min} = \min_{t \in [m]} \lambda_{\min}(\Sigma_t)$ and $\lambda_{\max} = \max_{t \in [m]} \lambda_{\max}(\Sigma_t)$ be the bound on the eigenvalues of these covariance matrices. Let $K$ be the maximal diagonal elements of the inverse convariance matrices:
\[
K = \max_{t \in [m]} \max_{j \in [p]} (\Sigma_{t}^{-1})_{jj},
\]

The following theorem is our main result, which is proved in appendix.
\begin{theorem}
\label{thm:support_recovery}
Suppose $\lambda$ in \eqref{eqn:local_lasso}
was chosen as $\lambda_t = 4\sigma \sqrt{\frac{\log p}{n}}$. Furthermore, suppose that the multi-task
coefficients in \eqref{eq:model} satisfy the following bound on the
signal strength
\begin{align}
\label{eqn:signal_strengh}
\min_{j \in S} 
 \sqrt{\sum_{t \in [m]} ( \beta^*_{tj})^2} 
\geq 
  6 K \sigma \sqrt{\frac{m + \log p}{n}} + \frac{C \sigma_X^4 \lambda_{\max}^{1/2}  \sigma |S| \sqrt{m} \log p}{\lambda_{\min}^{3/2} n}:= 2 \Lambda^*,
\end{align}
where $C < 5000$. Then the support estimated by the master node satisfies $\hat S(\Lambda^*)
= S$ with probability at least $1 - mp^{-1}$.
\end{theorem}

Let us compare the minimal signal strength to that required by the
lasso and group lasso.  Let $B =
[\beta_1,\beta_2,\ldots,\beta_m]\in\RR^{p \times m}$ be the matrix of
true coefficients.  Simplifying \eqref{eqn:signal_strengh}, we have
that our procedure requires the minimum signal strength to satisfy
\begin{equation}
  \label{eq:DSMLsigstrength}
  \min_{j \in S} \frac{1}{\sqrt{m}} \norm{B_j}_2 
\gtrsim \sqrt{\frac{1}{n} \rbr{1 + \frac{\log p }{m}}} + \frac{ |S| \log p}{n},
\end{equation}
where $a(n) \gtrsim b(n)$ means that for some $c,N$, $a(n) > c \cdot
b(n), \forall n > N$.  For the centralized group lasso, the standard
analysis assumes a stronger condition on the data, namely that the
design matrix satisfies mutual incoherence with parameter $\alpha$ and
sparse eigenvalue condition. Mutual incoherence is a much stronger
conditions on the design in comparison to the generalized coherence
condition required by DSML. Group lasso recovers the support if
\citep[Corollary 5.2 of][]{Lounici2011Oracle}:
\begin{equation}
  \label{eq:GLsigstrength}
  \min_{j \in S} \frac{1}{\sqrt{m}} \norm{B_j}_2 
\geq \frac{4\sqrt{2} C_{\alpha,\kappa} \sigma}{\sqrt{n}} \sqrt{1 + \frac{2.5 \log p}{m}} \gtrsim \sqrt{\frac{1}{n} \rbr{1 + \frac{\log p }{m}}}.
\end{equation}
where $C_{\alpha,\kappa}$ 
is some constant depend on the mutual incoherence and sparse eigenvalue parameters.  Under the irrepresentable condition
on the design (which is weaker than the mutual incoherence), the lasso
requires the signal to satisfy \citep{bunea08honest,wainwright06sharp}:
\begin{equation}
  \label{eq:LASSOsigstrength}
  \min_{t \in [m]} 
\min_{j \in S} 
|\beta_{tj}^*| \geq   C_{\gamma,\kappa} \sigma \sqrt{\frac{\log p}{n}} \gtrsim  \sqrt{\frac{\log p}{n}}
\end{equation}
for some constant $C_{\gamma,\kappa}$ of the mutual coherence
parameter $\gamma$ and of $\kappa$.  Ignoring for the moment
the differences in the conditions on the design matrix, there are two
advantages of the multitask group lasso over the local lasso: relaxing
the signal strength requirement to a requirement on the average
strength across tasks, and a reduction by a factor of $m$ on the $\log
p$ term.  Similarly to the group lasso, DSML requires a lower bound
only on the average signal strength, not on any individual
coefficient.  And as long as $m \ll n$, or more precisely $n \gtrsim
\frac{m|S|^2 (\log p)^2}{\kappa^2(m + \log p)}$ enjoys the same linear
reduction in the dominant term of the required signal strength, match
the leading term of the group lasso bound.

Based on Theorem~\ref{thm:support_recovery}, we have the following
corollary that characterizes estimation error and prediction risk of
DSML, with the proof  given in the appendix.
\begin{corollary}
\label{cor:prediction_estimation}
Suppose the conditions of Theorem~\ref{thm:support_recovery}
hold. With probability at least $1 - mp^{-1}$, we have
\begin{align*}
& \sum_{j=1}^p \norm{\tilde B_{j} - B_j}_2 \leq 6 K |S| \sigma \sqrt{\frac{m + \log p}{n}} + \frac{C \sigma_X^4 \lambda_{\max}^{1/2}  \sigma |S|^2 \sqrt{m} \log p}{\lambda_{\min}^{3/2} n} \\
\intertext{and}
&\frac{1}{nm} \sum_{t=1}^m (\EE_{X_t} (X_t^T \tilde \beta_t - X_t^T \beta^*_t))^2 \leq 
\frac{36 K^2 |S| \sigma^2}{n}\rbr{1 + \frac{\log p}{m}} + \frac{C^2 \sigma_X^8 \lambda_{\max}  \sigma^2 |S|^3 (\log p)^2}{\lambda_{\min}^{3} n^2}.
\end{align*}
\end{corollary}
Let us compare these guarantees for to the group lasso.  For DSML
Corollary 2 yields:
\begin{equation}
\frac{1}{\sqrt{m}} \sum_{j=1}^p \norm{\tilde B_{j} - B_j}_2 
\lesssim \frac{|S|}{\sqrt{n}} \sqrt{1 + \frac{\log p}{m}} + \frac{|S|^2 \log p}{ n},  
\end{equation}
When using the group lasso, the restricted eigenvalue condition is sufficient for obtaining error bounds and following holds for
the group lasso
\citep[Corollary 4.1 of][]{Lounici2011Oracle}:
\begin{equation}\label{eq:A}
\frac{1}{\sqrt{m}} \sum_{j=1}^p \norm{\tilde B_{j} - B_j}_2  \leq \frac{32\sqrt{2} \sigma |S|}{\kappa \sqrt{n}} \sqrt{\rbr{1 + \frac{2.5 \log p}{m}}} \lesssim \frac{|S|}{\sqrt{n}} \sqrt{1 + \frac{\log p}{m}},
\end{equation}
which is min-max optimal (up to a logarithmic factor).  Albeit with
the stronger generalized coherence condition, DSML matches this bound
when $n \gtrsim \frac{m|S|^2 (\log p)^2}{(m + \log p)}$.  Similarly
for prediction DSML attains:
\begin{equation}
\frac{1}{nm} \sum_{t=1}^m (\EE_{X_t} (X_t^T \tilde \beta_t - X_t^T \beta^*_t))^2
  \lesssim \frac{|S| \sigma^2}{n} \rbr{1 + \frac{\log p}{m}} + \frac{\sigma^2 |S|^3 (\log p)^2}{n^2},  
\end{equation}
which in the same regime matches the group lasso minimax optimal rate:
\begin{equation}\label{eq:B}
\frac{1}{nm} \sum_{t=1}^m (\EE_{X_t} (X_t^T \tilde \beta_t - X_t^T \beta^*_t))^2 
 \leq \frac{128  |S| \sigma^2}{\kappa n} \rbr{1 + \frac{2.5 \log p}{m}}  \lesssim \frac{|S| \sigma^2}{n} \rbr{1 + \frac{\log p}{m}}.  
\end{equation}
In both cases, as long as $m$ is not too large, we have a linear improvement
over Lasso, which corresponds to \eqref{eq:A} and \eqref{eq:B} with $m=1$.

\section{Experimental results}

\begin{figure}[t]
\begin{center}
\includegraphics[width=0.33 \textwidth]{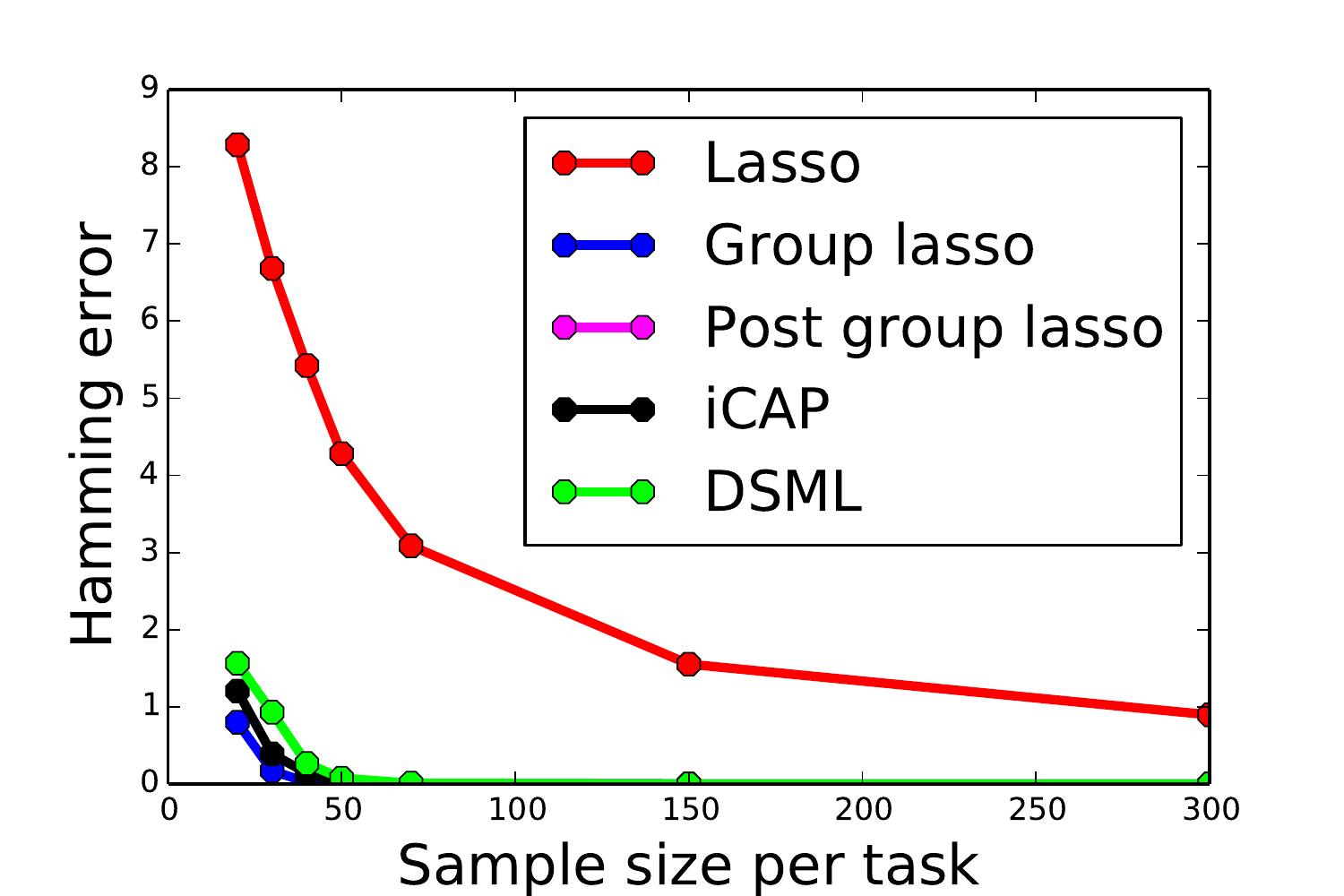}%
\includegraphics[width=0.33 \textwidth]{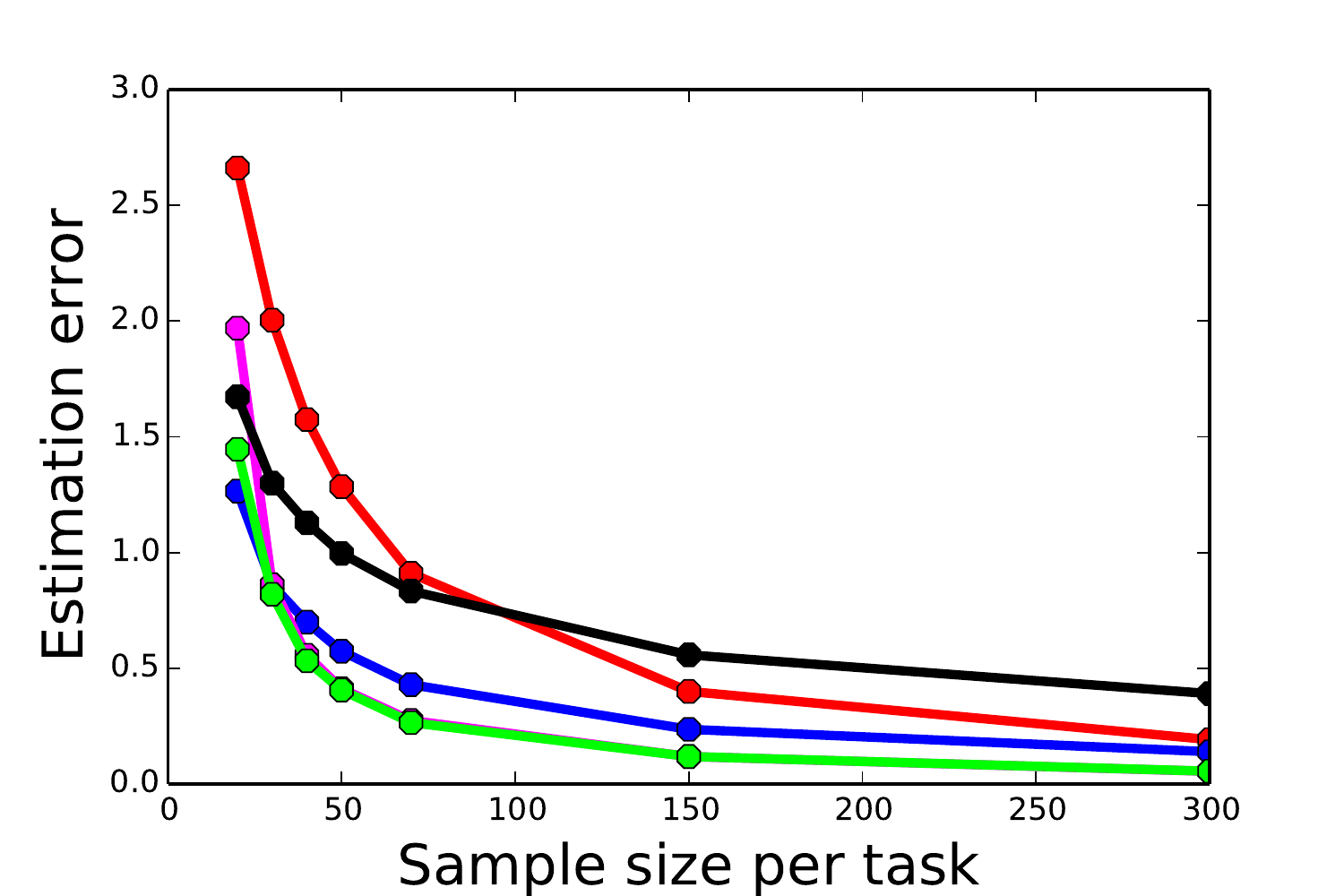}%
\includegraphics[width=0.33 \textwidth]{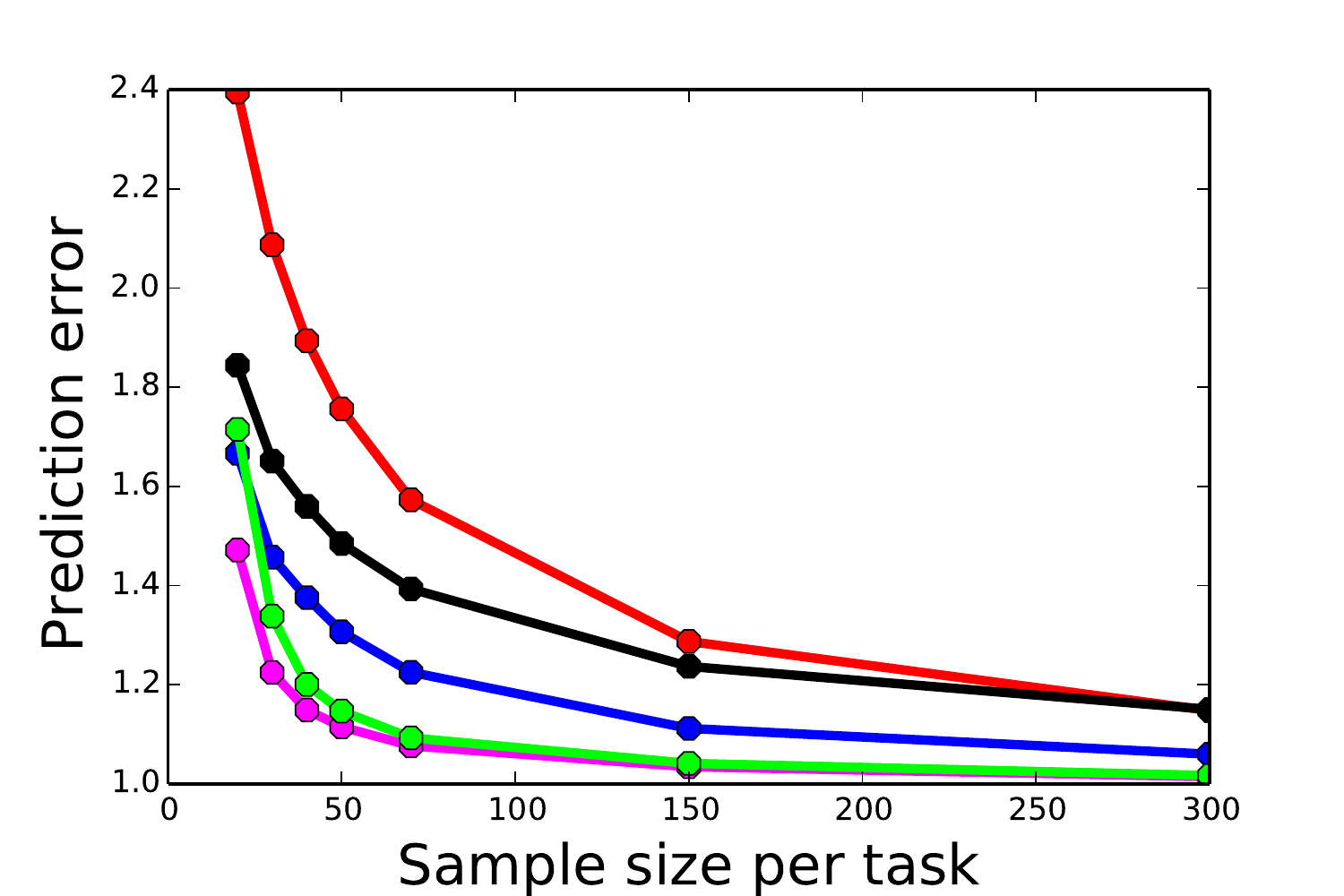}%
\end{center}
\begin{center}
\includegraphics[width=0.33 \textwidth]{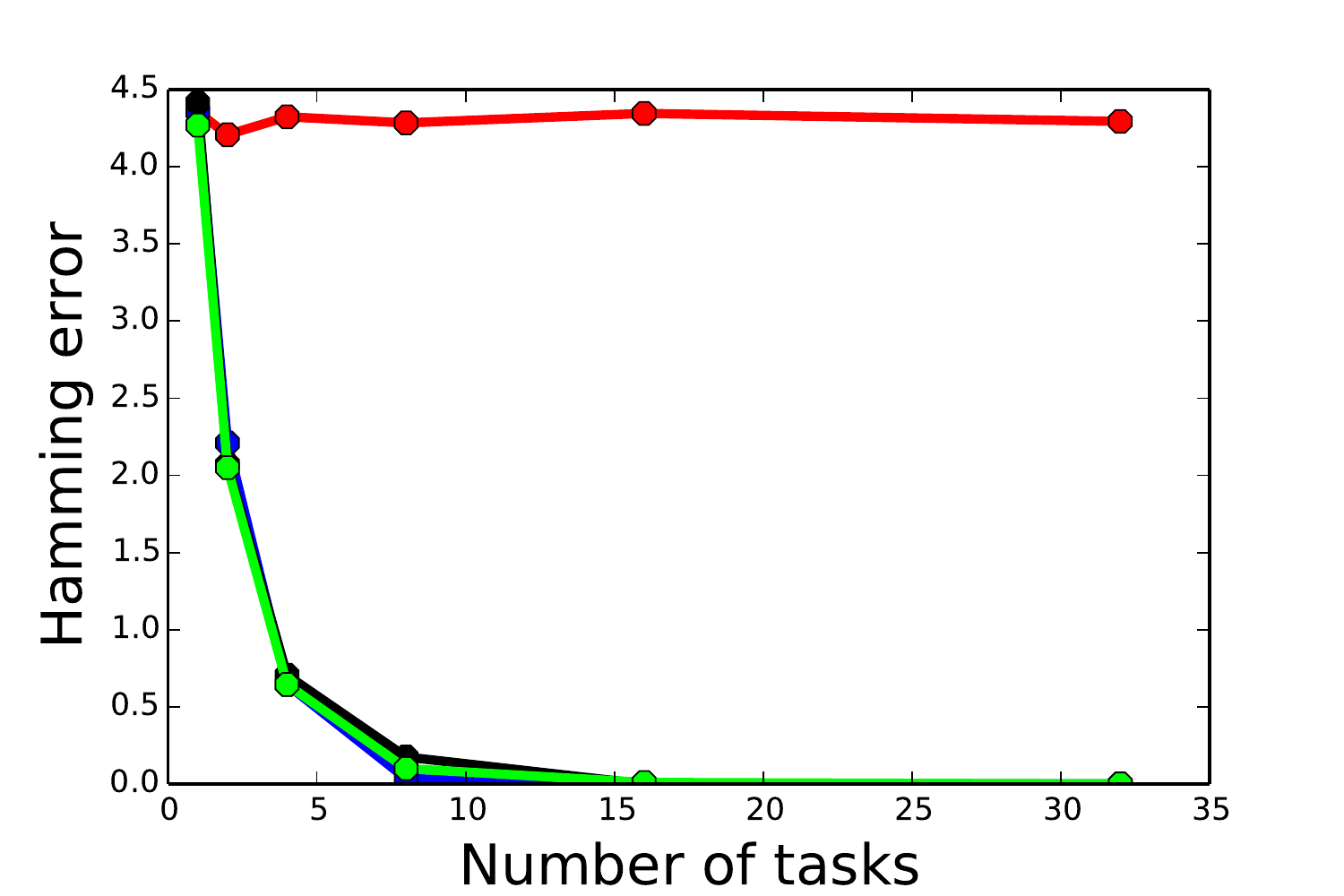}%
\includegraphics[width=0.33 \textwidth]{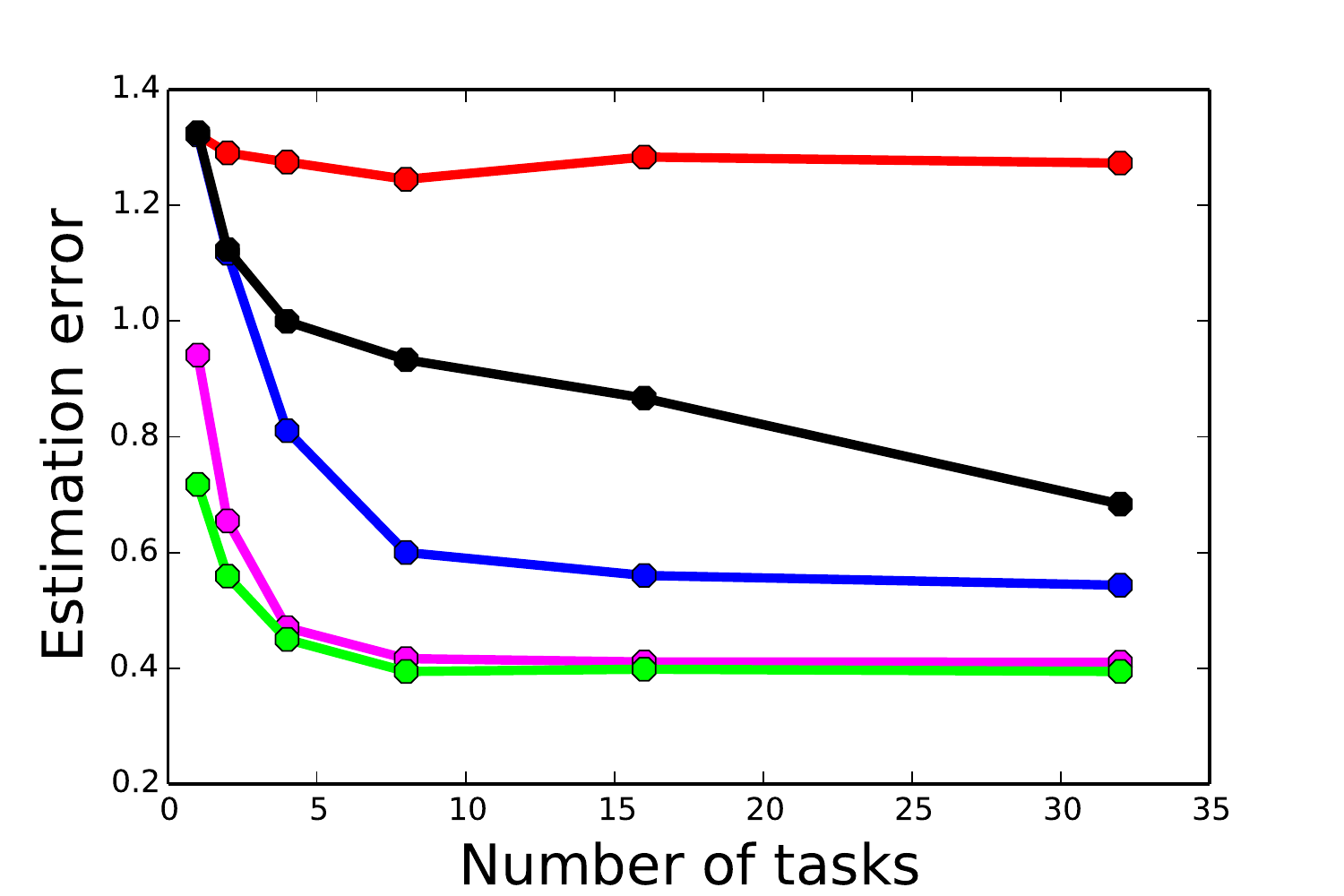}%
\includegraphics[width=0.33 \textwidth]{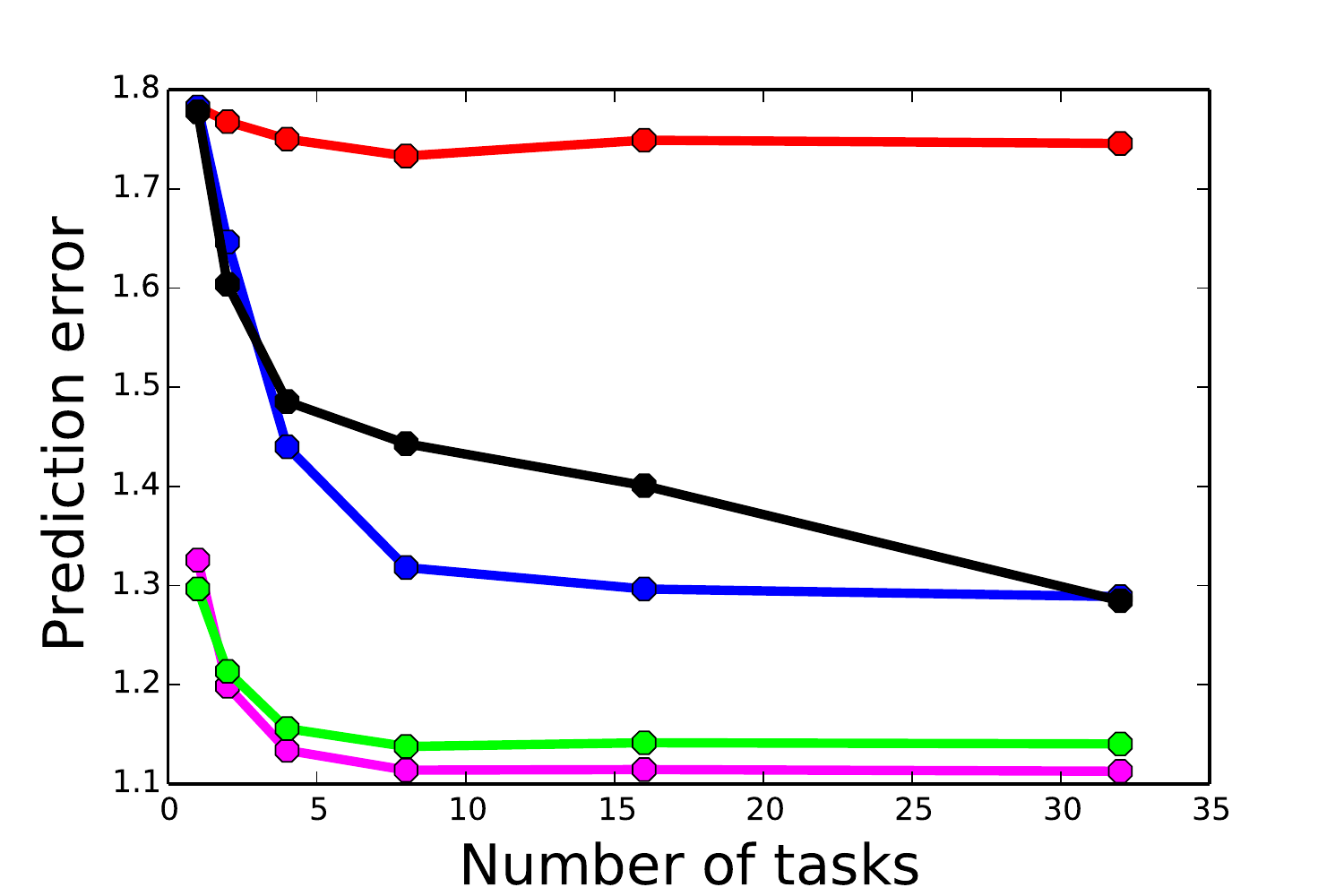}%
\end{center}
\caption{ Hamming distance, estimation error, and prediction error for
  multi-task regression with $p = 200$. Top row: the number of tasks
  $m=10$. Sample size per tasks is varied. Bottom row: Sample size
  $n=50$. Number of tasks $m$ varied.}
\label{fig:simulation_regression}
\end{figure}

Our first set of experiments is on simulated data. We generated
synthetic data according to the model in \eqref{eq:model} and in
\eqref{eq:logistic_regression_model}. Rows of $X_t$ are sampled from a
mean zero multivariate normal with the covariance matrix $\Sigma =
(\Sigma_{ab})_{a,b\in[p]}$, $\Sigma_{ab} = 2^{-|a-b|}$. The data
dimension $p$ is set to $200$, while the number of true relevant
variables $s$ is set to $10$.  Non-zero coefficients of $\beta$ are
generated uniformly in $[0,1]$.  Variance $\sigma^2$ is set to 1. Our
simulation results are averaged over 200 independent runs.

We investigate how performance of various procedures changes as a
function of problem parameters $(n,p,m,s)$. We compare the following
procedures: i) local lasso, ii) group lasso, iii) refitted group
lasso, where a worker node performs ordinary least squares on the
selected support, iv) iCAP, and v) DSML. The parameters for local
lasso, group lasso and iCAP were tuned to achieve the minimal Hamming
error in variable selection. For DSML, to debias the output of local
lasso estimator, we use $\mu = \sqrt{\log p/n}$. The thresholding
parameter $\Lambda$ is also optimized to achive the best variable
selection performance. The simulation results for regression are shown
in Figure \ref{fig:simulation_regression}. In terms of
support recovery (measured by Hamming distance), Group lasso, iCAP,
and DSML all perform similarly and significantly better than the local
lasso. In terms of estimation error, lasso perform the worst, while
DSML and refitted group lasso perform the best. This might be a result
of bias removal introduced by regularization.  Since the
group lasso recovers the true support in most cases, refitting on
it yields the maximum likelihood estimator on the true support.  It is
remarkable that DSML performs almost as well as this oracle estimator.

Figure \ref{fig:simulation_classification} shows the simulation results for classification. Similar with the regression case, we make the following observations:
\begin{itemize}
\item The group sparsity based approaches, including DSML, significantly outperform the individual lasso.
\item In terms of Hamming variable selection error, DSML performs slightly worse than group lasso and iCAP. While in terms of estimation error and prediction error, DSML performs much better than group lasso and icap. Given the fact that group lasso recovers the true support in most cases, refitted group lasso is equivalent to oracle maximum likelihood estimator. It is remarkable that DSML only performs slightly worse than refitted group lasso.
\item The advantage of DSML, as well as group lasso over individual lasso, becomes more and more significant with the increase in number of tasks.
\end{itemize}

\begin{figure}[htpb]
\begin{center}
\includegraphics[width=0.33 \textwidth]{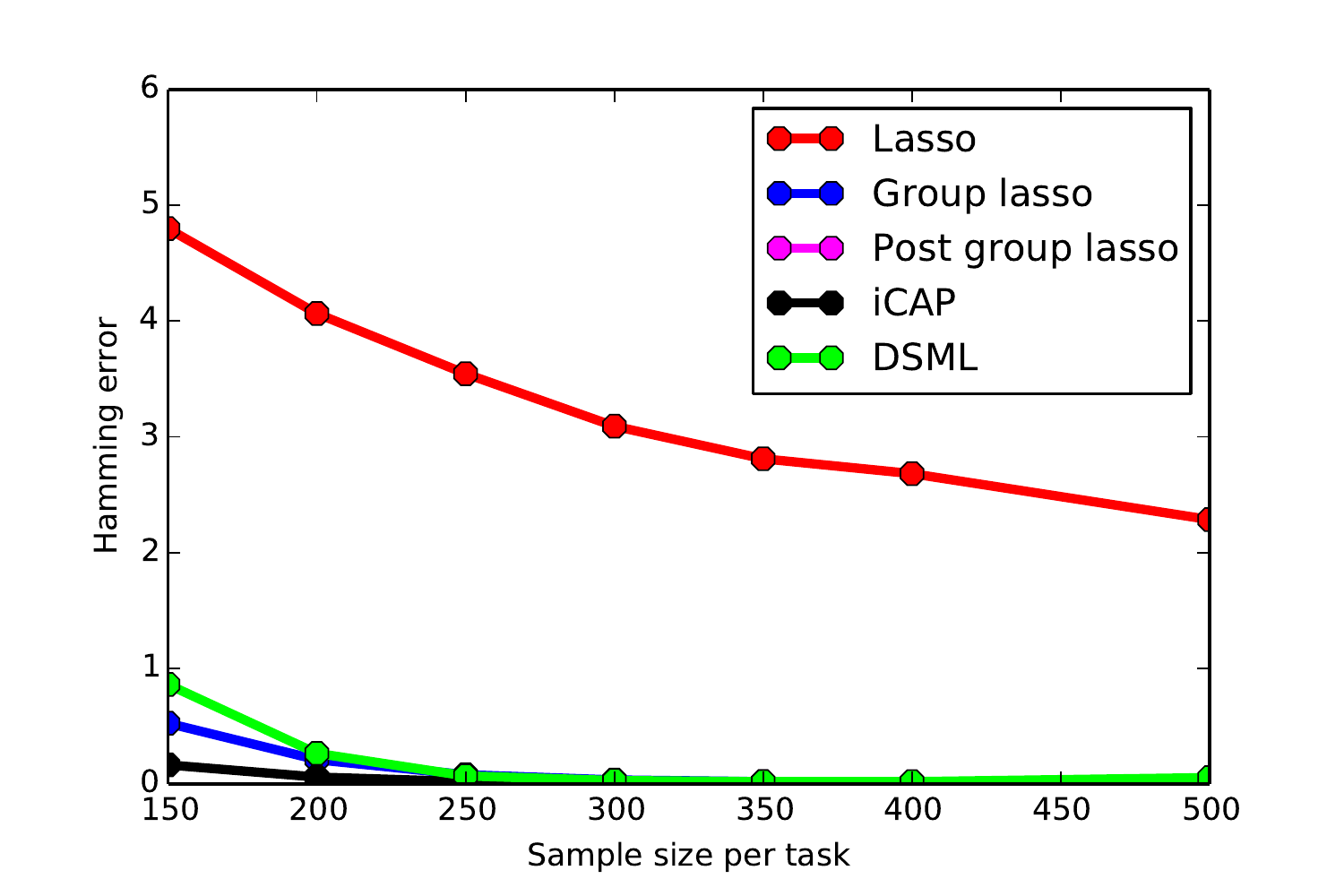}%
\includegraphics[width=0.33 \textwidth]{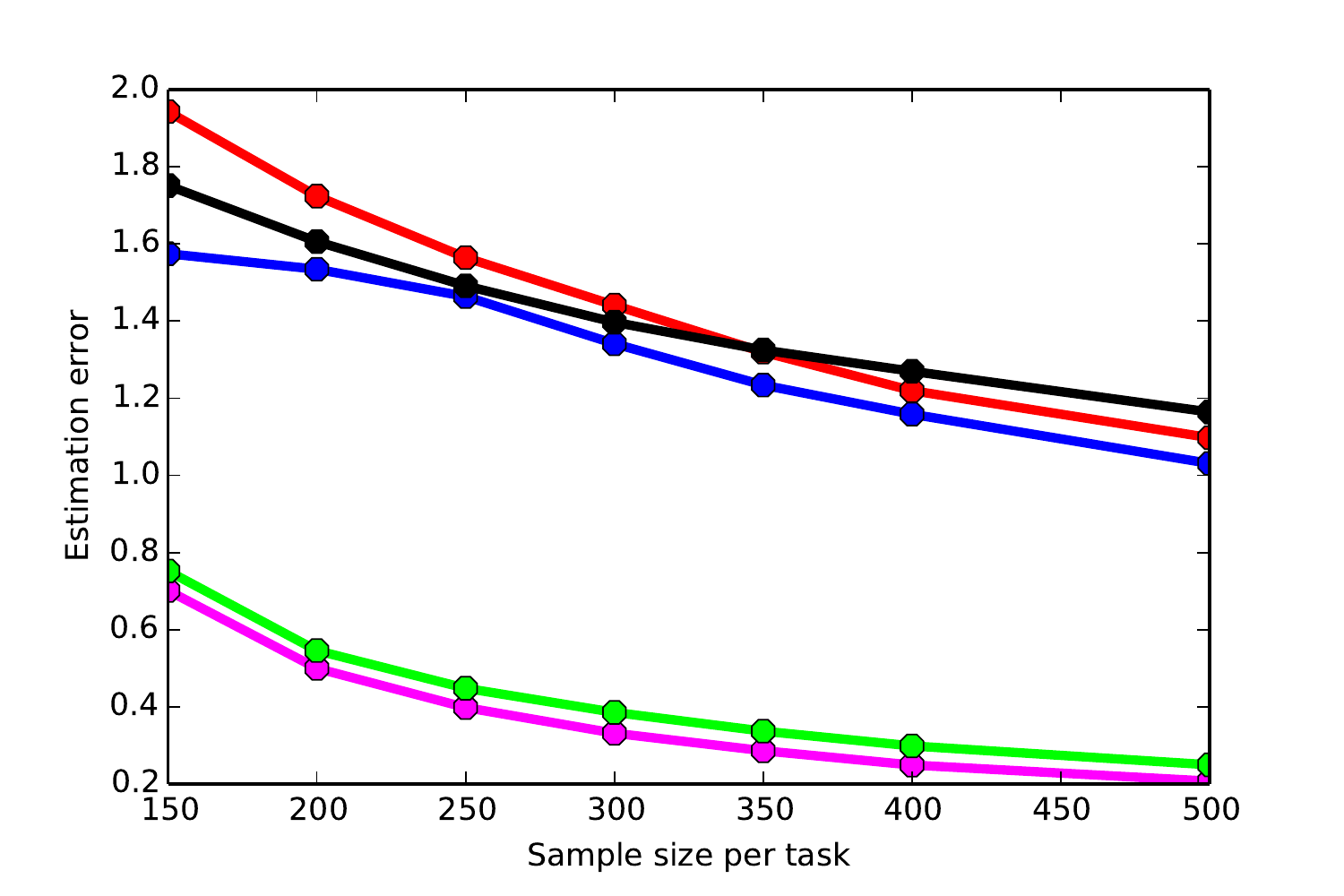}%
\includegraphics[width=0.33 \textwidth]{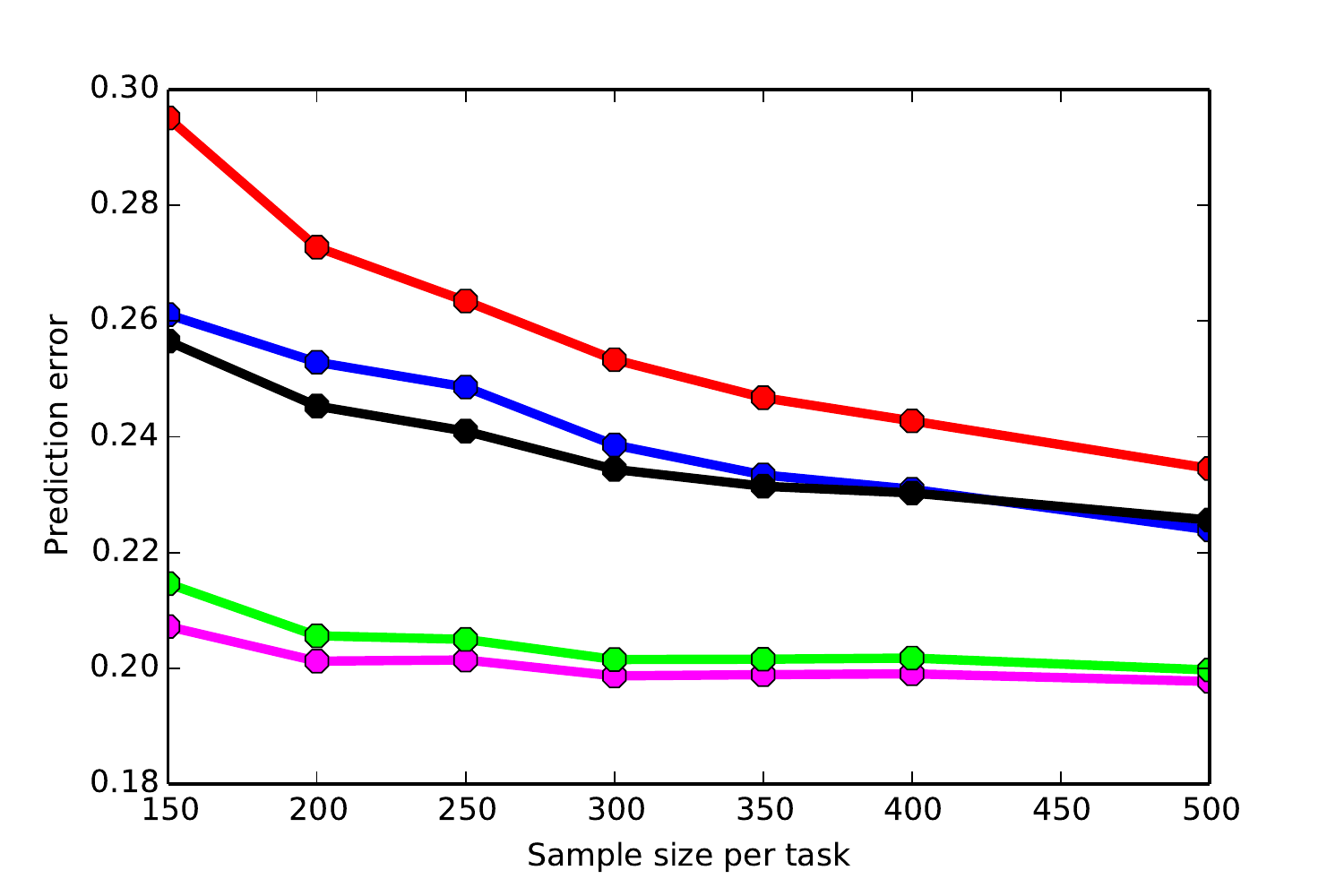}%
\end{center}
\begin{center}
\includegraphics[width=0.33 \textwidth]{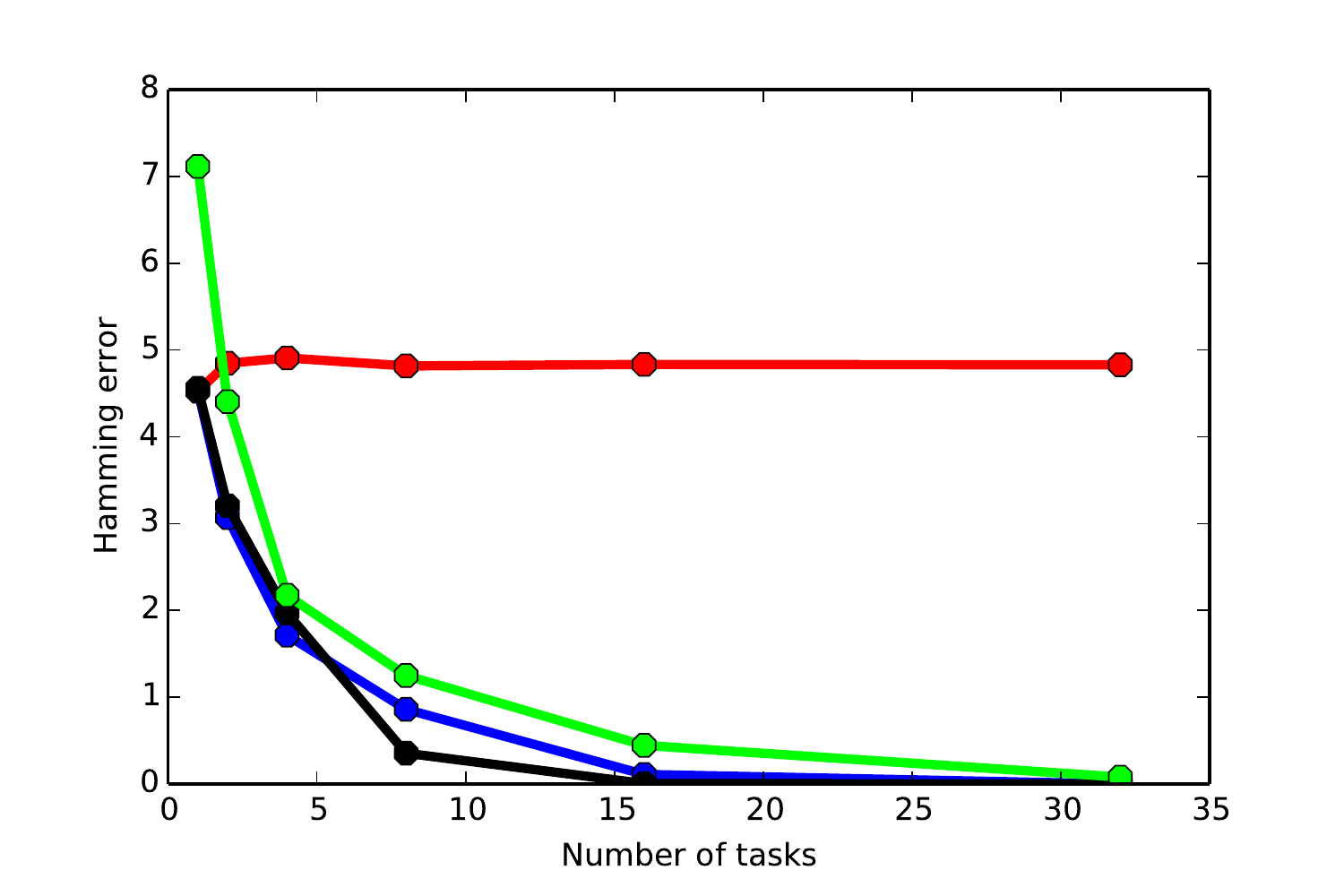}%
\includegraphics[width=0.33 \textwidth]{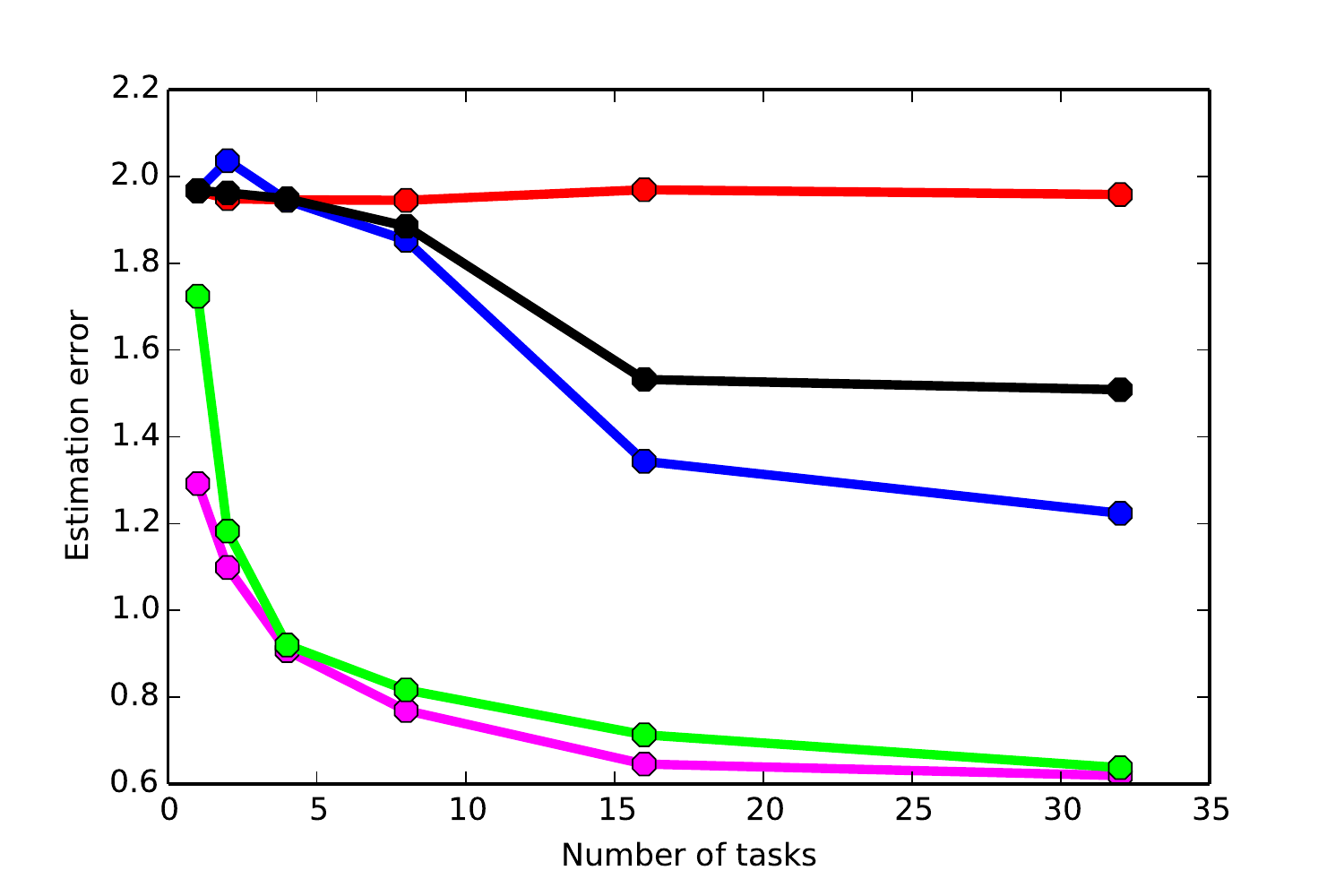}%
\includegraphics[width=0.33 \textwidth]{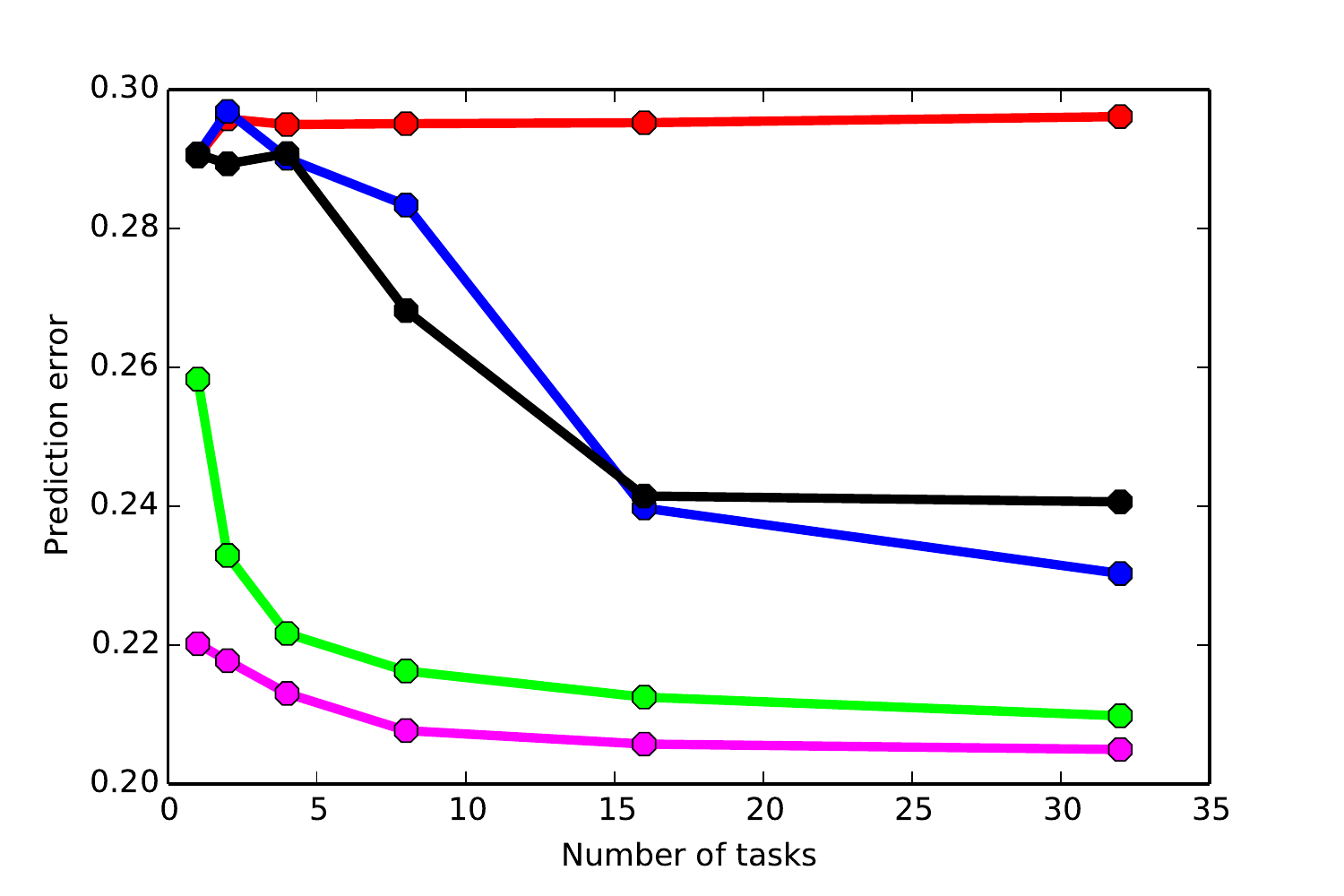}%
\end{center}
\caption{ Hamming distance, estimation error, and prediction error for
  multi-task classification with $p = 200$. Top row: the number of tasks
  $m=10$. Sample size per tasks is varied. Bottom row: Sample size
  $n=150$. Number of tasks $m$ varied.}
\label{fig:simulation_classification}
\end{figure}

We also evaluated DSML on the following benchmark data
sets considered in previous investigations of shared support multi-task learning:\begin{itemize}
\item[\bf School.] This is a widely used dataset for multi-task
  learning \citep{argyriou08convex}. The goal is to predict the students' performance at
  London's secondary schools. There are 27 attributes for each
  student. The tasks are naturally divided according to different
  schools. We only considered schools with at least 200 students,
  which results in 11 tasks.
\item[\bf Protein.] The task is to predict the protein secondary
  structure \citep{proteindata}. We considered three binary classification tasks here: coil vs helix, helix vs strand, strand vs coil. The dataset consists of 24,387 instances in total, each with 357 features.
\item[\bf OCR.]
  We consider the optical character recognition problem. Data were gathered by Rob Kassel at the MIT Spoken Language Systems Group
  \footnote{\url{http://www.seas.upenn.edu/~taskar/ocr/}}. Following \citep{DBLP:journals/sac/ObozinskiTJ10}, we consider the following 9 binary classification task: c vs
  e, g vs y, g vs s, m vs n, a vs g, i vs j, a vs o, f vs t, h vs
  n. Each image is represented by $8 \times 16$ binary pixels.
\item[\bf MNIST.]  This is a handwritten digit recognition dataset \footnote{\url{http://yann.lecun.com/exdb/mnist/}}, the    ata consists of images that represent digits. Each
  image is represented by 784 pixels.  We considered the following 5
  binary classification task: 2 vs 4, 0 vs 9, 3 vs 5, 1 vs 7, 6 vs 8.
\item[\bf USPS.] This dataset consists handwritten images from envelopes by the U.S. Postal Service. We considere the following 5 binary classification task: 2 vs 4, 0 vs 9, 3 vs 5, 1 vs 7, 6 vs 8. Each image is represented by
  256 pixels.
\item[\bf Vehicle.] We considered the vehicle classification
  problem in distributed sensor networks \citep{Duarte:2004:VCD:1034812.1034817}. We considered the following
  3 binary classification task: AAV vs DW, AAV vs noise, DW vs
  noise. There are 98,528 instances in total, each instances is
  described by 50 acoustic features and 50 seismic features.
\end{itemize}

\begin{figure}[htpb]
\begin{center}
\includegraphics[width=0.33 \textwidth]{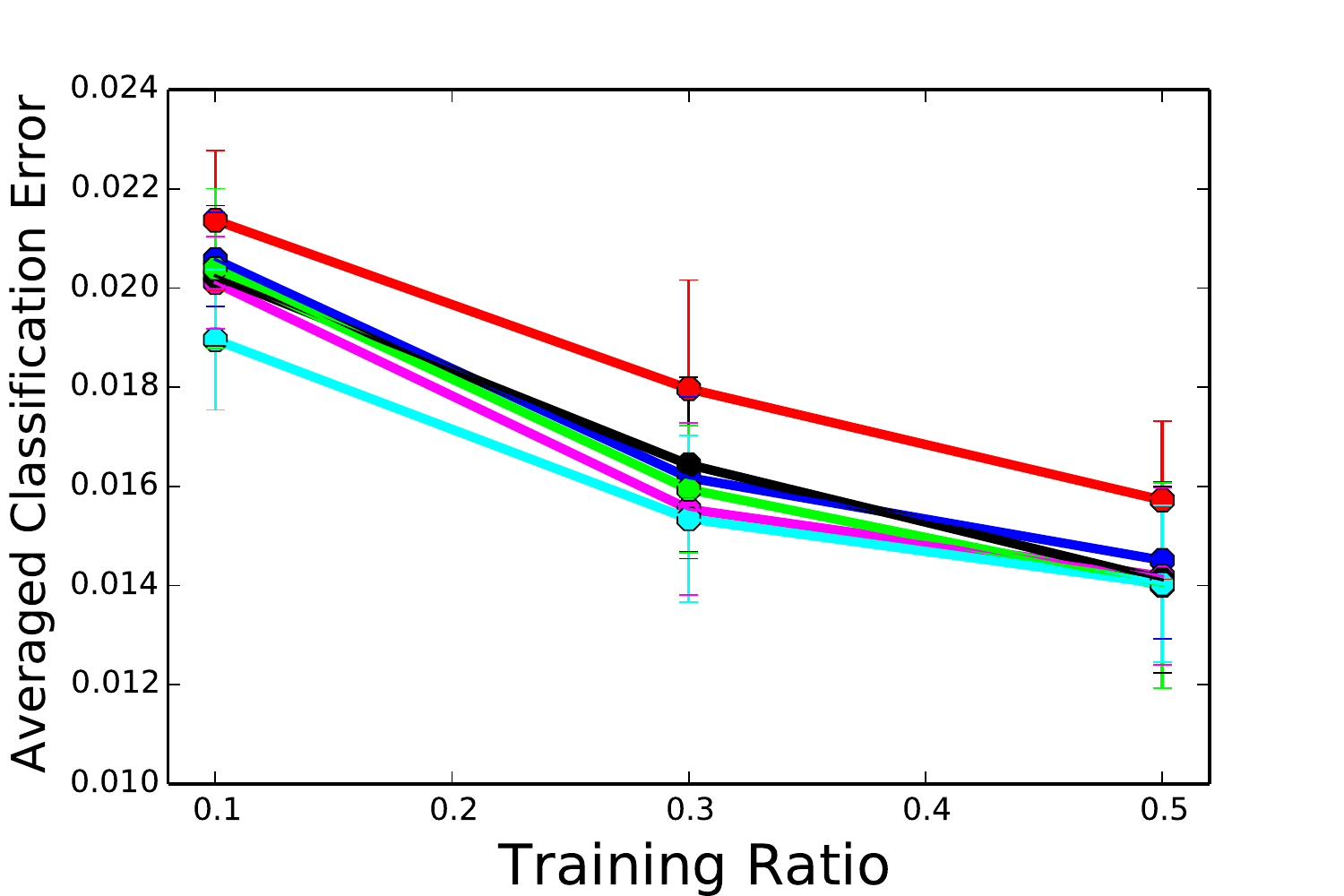}%
\includegraphics[width=0.33 \textwidth]{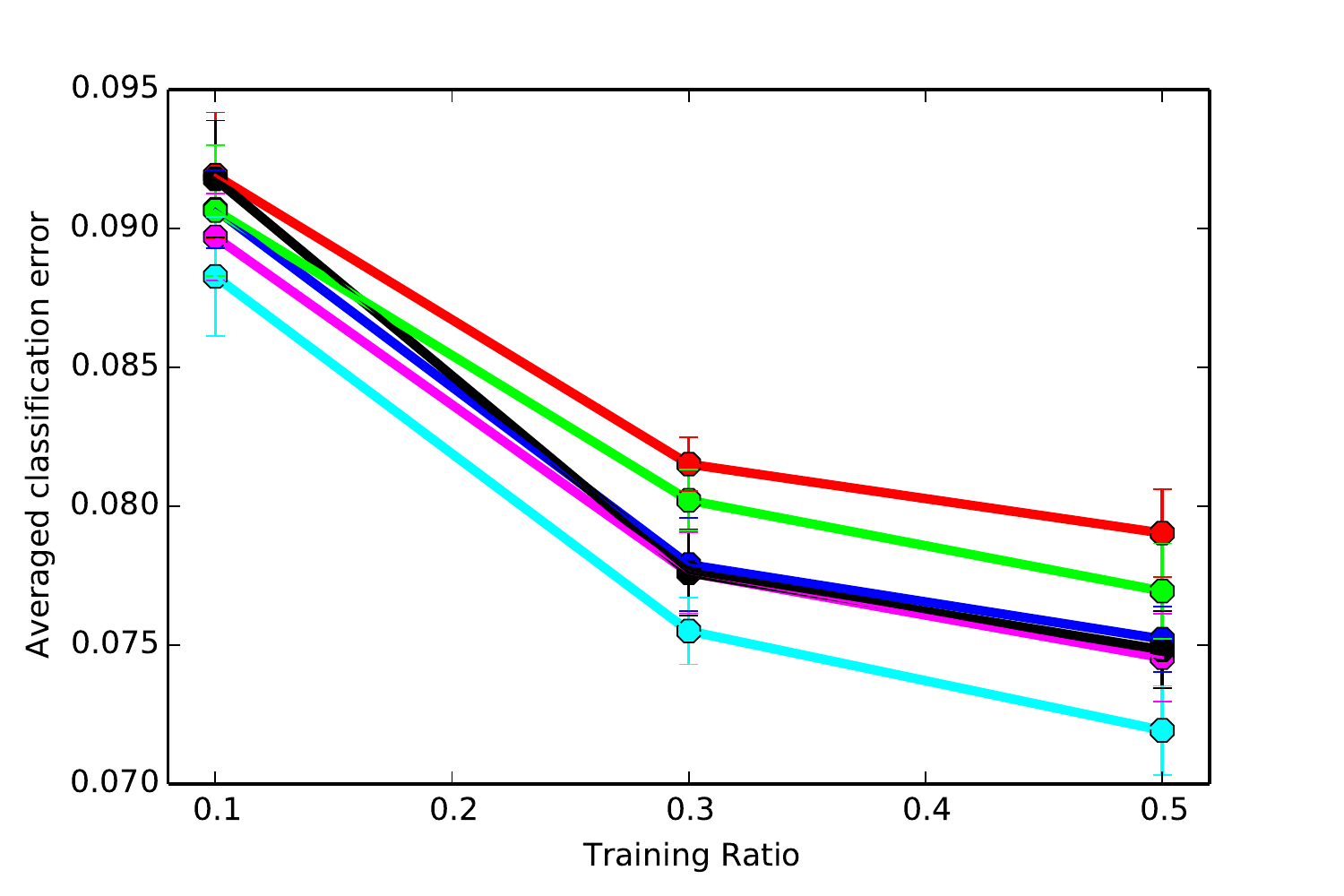}%
\includegraphics[width=0.33 \textwidth]{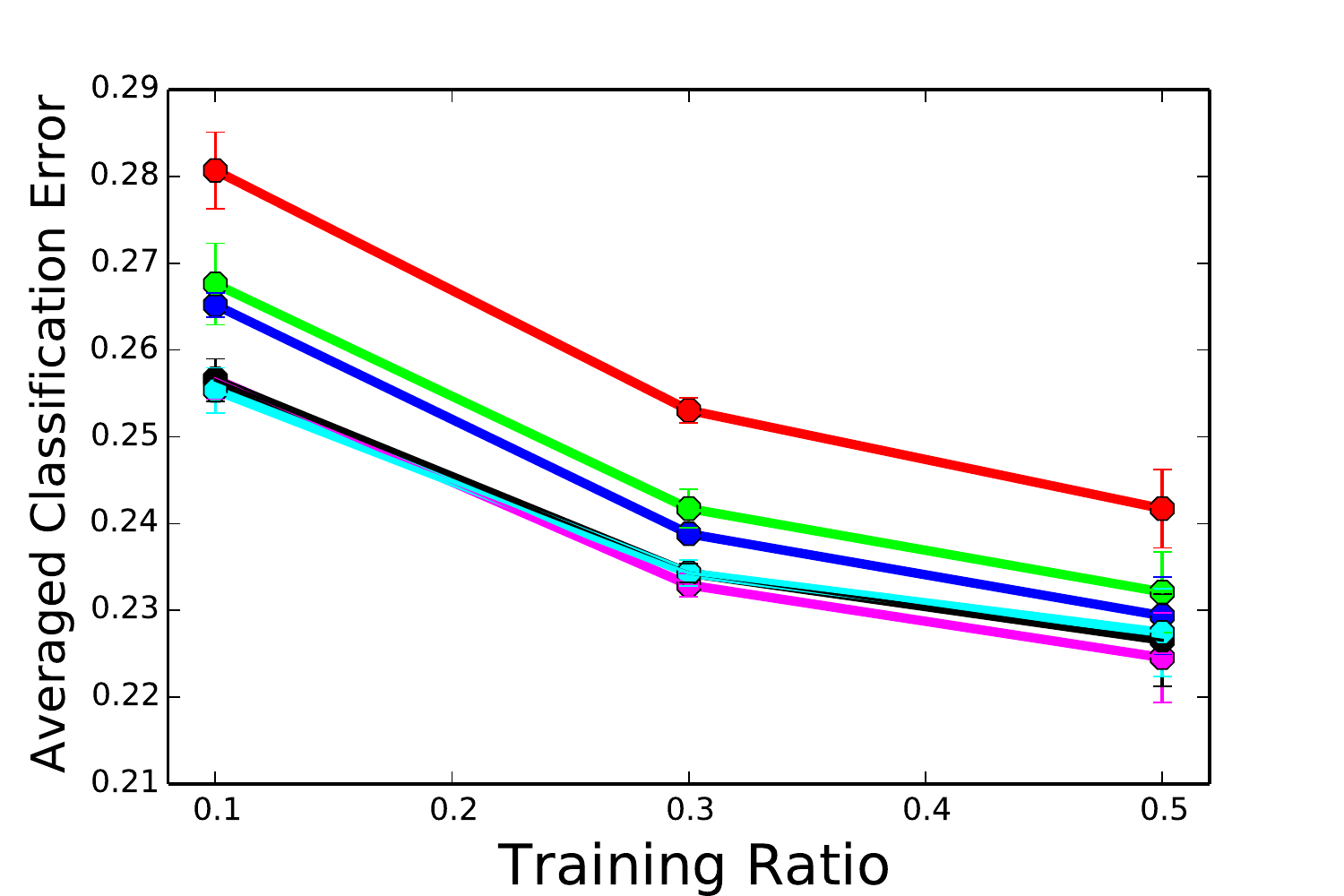}%
\end{center}
\makebox[0.33 \textwidth]{\textbf{USPS}}\makebox[0.33 \textwidth]{\textbf{OCR}}\makebox[0.33 \textwidth]{\textbf{Protein}}
\begin{center}
\includegraphics[width=0.33 \textwidth]{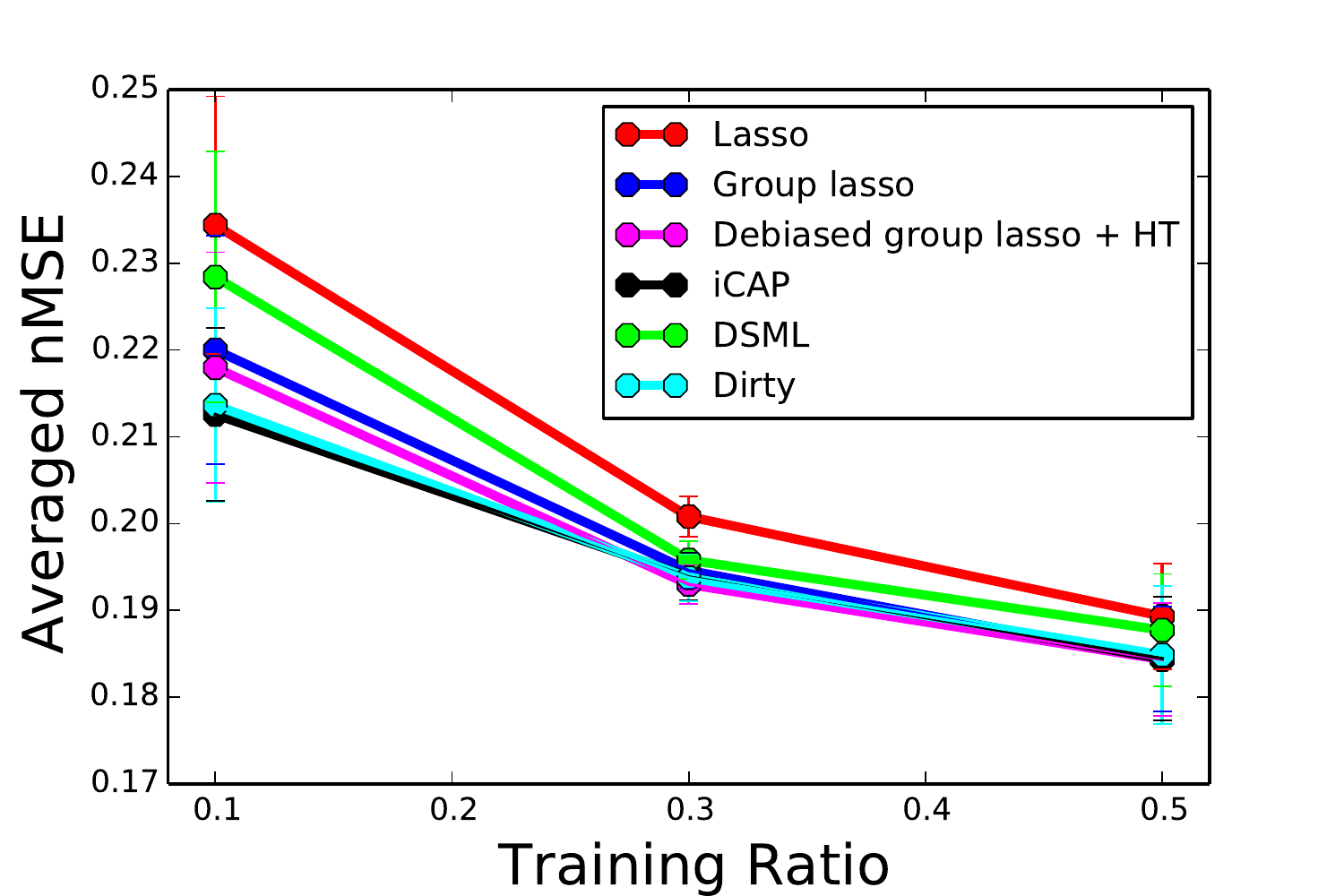}%
\includegraphics[width=0.33 \textwidth]{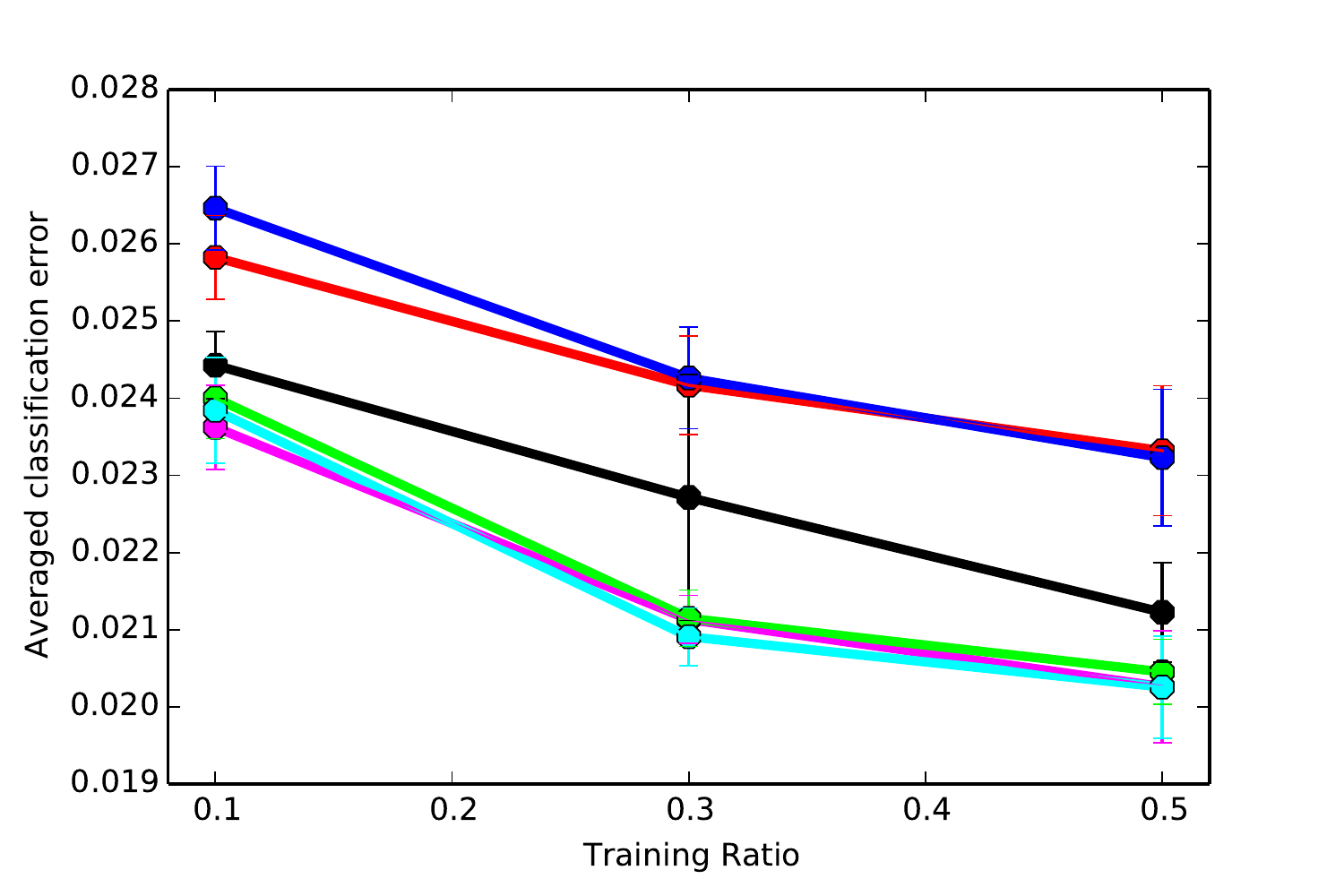}%
\includegraphics[width=0.33 \textwidth]{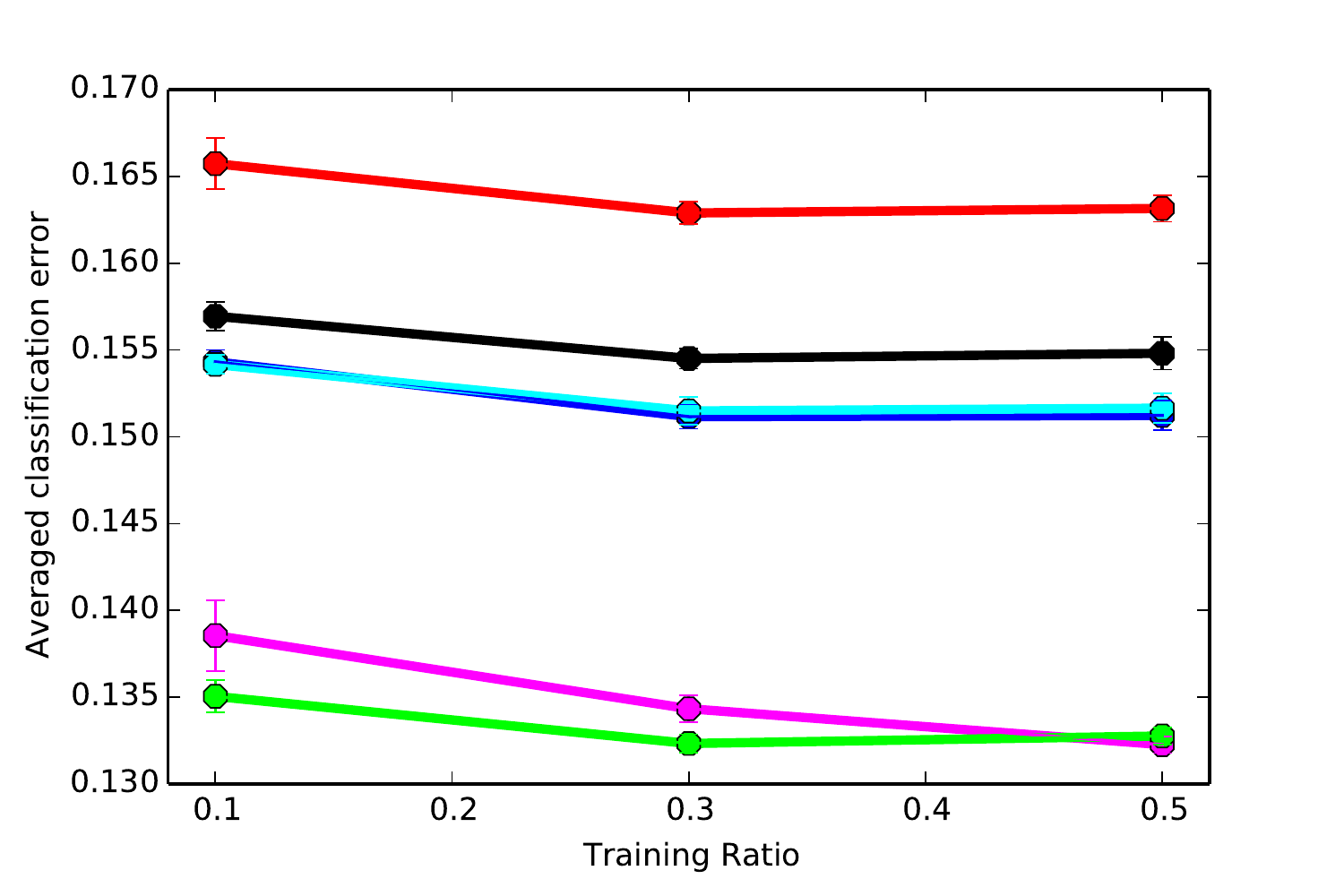}%
\end{center}
\makebox[0.33 \textwidth]{\textbf{School}}\makebox[0.335 \textwidth]{\textbf{MNIST}}\makebox[0.33 \textwidth]{\textbf{Vehicle}}
\caption{Comparison on real world datasets.}
\label{fig:real_data}
\end{figure}

In addition to the procedures used in the previous section, we also
compare against the dirty model \citet{jalali2010dirty}, as well as
the centralized approach that first debiases the group lasso and then
performs group hard thresholding as in \eqref{eq:hard_thr}.
Regularization and thresholding parameters were tuned on a held-out
set consisting of $20\%$ of the data.  In Figure \ref{fig:real_data}
we report results of training on $10\%$, $30\%$ and, $50\%$ of the
total data set size.  The multi-task methods clearly preform better
than the local lasso, with DSML achieving similar error as the
centralized methods.

\section{Discussion}

We introduced and studied a shared-sparsity distributed multi-task
learning problem.  We presented a novel communication-efficient
approach that required only one round of communication and achieves
provable guarantees that compete with the centralized approach to
leading order up to a generous bound on the number of machines.  Our
analysis was based on Restricted Eigenvalue and Generalized Coherence
conditions.  Such conditions, or other similar conditions, are
required for support recovery, but much weaker conditions are
sufficient for obtaining low prediction error with the lasso or group
lasso.  An interesting open question is whether there exists a
communication efficient method for distributed multi-task learning
that requires sample complexity $n=O(|S|+(\log p)/m)$, like the group
lasso, even without Restricted Eigenvalue and Generalized Coherence
conditions, or whether beating the $n=O(|S|+\log p)$ sample complexity
of the lasso in a more general setting inherently requires large
amounts of communication.  Our methods, certainly, rely on these
stronger conditions.

DSML can be easily extended to other types of structured sparsity,
including sparse group lasso \citep{friedman10note}, tree-guided group
lasso \citep{kim2010tree} and the dirty model \citep{jalali2010dirty}.
Going beyond shared sparsity, shared subspace (i.e.~low rank) and
other matrix-factorization and feature-learning methods are also
commonly and successfully used for multi-task learning, and it would
be extremely interesting to understand distributed multi-task learning
in these models.

\bibliographystyle{plainnat}
\bibliography{paper}

\section{Appendix}
\appendix
\section{Proof of Theorem 1}
We first introduce the following lemma.
\begin{lemma}
When the rows of $X_1,\ldots,X_t$ are independent subgaussian random vectors, with mean zero, covariance $\Sigma_1,...,\Sigma_t$, respectively. Let 
\[
C_M = \max_{t \in [m]} \max_{j \in [p]} \rbr{M_t^T \rbr{\frac{X_t^T X_t}{n}} M_t }_{jj}. 
\]
Then with probability at least $1 - 2mp \exp{(- cn)} - 2mp^{-2}$ for some constant $c$, we have
\[
C_M \leq 2 \max_{t \in [m]} \max_{j \in [p]} (\Sigma_{t}^{-1})_{jj}.
\]
\end{lemma}
\begin{proof}
As shown in Theorem 2.4 of \citep{Javanmard2013Confidence}, $\Sigma_t^{-1}$ will be a feasible solution for the problem of estimating $M_t$. Since we're minimizing ${(M_t^T \hat \Sigma_t M_t)}_{jj}$, we must have
\[
\max_{j \in [p]} {(M_t^T \hat \Sigma_t M_t)}_{jj} \leq  \max_{j \in [p]} {(\Sigma_t^{-1} \hat \Sigma_t \Sigma_t^{-1})}_{jj}.
\]
Based on the concentration results of sub-exponential random variable \citep{Vershynin2012Introduction}, also Lemma 3.3 of \citep{lee2015communication}, we know with probability at least $1 - 2p \exp{(- cn)}$ for some constant $c$, we have
\[
\max_{j \in [p]} {(\Sigma_t^{-1} \hat \Sigma_t \Sigma_t^{-1})}_{jj} \leq 2 \max_{j \in [p]} (\Sigma_t^{-1})_{jj}.
\]
Take an union bound over $t \in [m]$, we obtain with probability at least $1 - 2mp \exp{(- cn)}$,
\[
C_M \leq \max_{t \in [m]} \max_{j \in [p]} {(M_t^T \hat \Sigma_t M_t)}_{jj} \leq  \max_{t \in [m]} \max_{j \in [p]} {(\Sigma_t^{-1} \hat \Sigma_t \Sigma_t^{-1})}_{jj} \leq 2 \max_{t \in [m]} \max_{j \in [p]} (\Sigma_{t}^{-1})_{jj}.
\]
\end{proof}

Now we are ready to prove Theorem 1, recall the model assumption
\begin{equation}
  \label{eq:model}
  y_t = X_t \beta^*_t + \epsilon_t,\qquad
  t=1,\ldots,m,
\end{equation}
and the debiased estimation
\begin{align}
\label{eqn:debiasing}
\hat \beta_t^u = 
\hat \beta_t + n^{-1} M_t X_t^T (y_t - X_t \hat \beta_t ),
\end{align}
we have
\begin{align*}
\hat \beta^u_t =& 
\hat \beta_t + \frac{1}{n} M_t X_t^T ( X_t \beta^*_t - X_t \hat \beta_t ) + \frac{1}{n} M_t X_t^T \epsilon_t \\
=& \beta_t^* + (M_t \hat \Sigma_t - I)( \beta^*_t - \hat \beta_t) + \frac{1}{n} M_t X_t^T \epsilon_t.
\end{align*}
For the term $(M_t \hat \Sigma_t - I)(\beta^*_t - \hat \beta_t)$, 
define 
\[
C_{\mu} = 10 e \sigma_X^4 \sqrt{\frac{\lambda_{\max}}{\lambda_{\min}}} ,
\]
we
have the following bound
\begin{equation}
\begin{aligned}
\norm{(M_t \hat \Sigma_t - I)(\beta^*_t - \hat \beta_t)}_{\infty} 
\leq& \max_j \norm{\hat \Sigma_t m_{tj} - e_j }_{\infty} \norm{\beta^*_t - \hat \beta_t}_1  \\
\leq_P& C_{\mu} \sqrt{\frac{\log p}{n}} \cdot \frac{16 A}{\kappa} \sigma |S| \sqrt{\frac{\log p}{n}} \\
=& \frac{16 A C_{\mu} \sigma |S| \log p}{\kappa n}.
\label{eqn:term_bias}
\end{aligned}
\end{equation}
Noticed that 
\[
n^{-1} M_t X_t^T \epsilon_t \sim \Ncal \left(0,\frac{\sigma^2 M_t \hat \Sigma_t {M_t}^T}{n} \right).
\]
Our next step uses a result on the concentration of $\chi^2$ random
variables. For any coordinate $j$, we have
\[
\sum_{i=1}^m \rbr{ n^{-1} e_j^T M_t X^T \epsilon_t }^2 \leq \frac{C_M^2 \sigma^2}{n} \sum_{i=1}^m \xi_i^2,
\]
where $(\xi_i)_{i \in [m]}$ are standard normal random
variables. Using Lemma \ref{lemma:chi_squared} with a weight vector 
\[
v = \rbr{\frac{C_M^2 \sigma^2}{n}, \frac{C_M^2 \sigma^2}{n}, \ldots, \frac{C_M^2 \sigma^2}{n}}
\]
and choosing $t = \sqrt{m} + \frac{\log p}{\sqrt{m}}$, we have
\[
P \cbr{ \frac{ \rbr{\frac{C_M^2 \sigma^2}{n}} \sum_{i=1}^m \xi_i^2}{\sqrt{2m} \rbr{\frac{C_M^2 \sigma^2}{n}}} - \sqrt{\frac{m}{2}} > \sqrt{m} + \frac{\log p}{\sqrt{m}}} \leq 2 \exp \rbr{ - \frac{\rbr{\sqrt{m} + \frac{\log p}{\sqrt{m}}}^2}{2 + 2 \sqrt{2} (1 + \frac{\log p}{m}) } }.
\]
A union bound over all $j \in [p]$ gives us that with  probability at least $1 - p^{-1}$
\begin{align}
\sum_{i \in [m]}
\rbr{ n^{-1}e_j^T M_t X^T \epsilon_t }^2 
\leq 3m \rbr{\frac{C_M^2 \sigma^2}{n}} + \sqrt{2} \log p \rbr{\frac{C_M^2 \sigma^2}{n}}, 
\qquad \forall j \in [p].
\label{eqn:term_chisquared}
\end{align}
Combining \eqref{eqn:term_bias} and \eqref{eqn:term_chisquared}, 
we get the following estimation error bound:
\begin{equation}
  \begin{aligned}
\norm{\hat B_{j} - B_j}_2 
=& \sqrt{\sum_{i \in [m]} \left([ M_t \hat \Sigma_t - I)(\beta^*_t - \hat \beta_t)]_j + \sbr{n^{-1} M_t X_t^T \epsilon_t }_j  \right)^2}
 \\
\leq& \sqrt{\sum_{i \in [m]} 2 \left([ M_t \hat \Sigma_t - I)(\beta^*_t - \hat \beta_t)]_j^2 +\sbr{n^{-1} M_t X_t^T \epsilon_t }_j^2  \right)}  \\
\leq& \sqrt{\sum_{i \in [m]} \left( \frac{512 A^2 C_{\mu}^2 \sigma^2 |S|^2 (\log p)^2}{\kappa^2 n^2}\right) + 6m \rbr{\frac{C_M^2 \sigma^2}{n}} + 2\sqrt{2} \log p \rbr{\frac{C_M^2 \sigma^2}{n}} }  \\
=& \frac{\sigma}{\sqrt{n}} \sqrt{\frac{512 A^2 C_{\mu}^2 m |S|^2 (\log p)^2 }{\kappa^2 n} + 6 C_M^2 m + 2\sqrt{2} C_M^2 \log p  } \\
\leq& \frac{91 C_{\mu} \sigma |S| \sqrt{m} \log p }{\kappa n} + 3 C_M \sigma \sqrt{\frac{m + \log p}{n}},
\label{eqn:bound_sc}
\end{aligned}
\end{equation}
where the first inequality uses the fact $(a+b)^2 \leq 2a^2 + 2b^2$,
and the second inequality uses \eqref{eqn:term_bias} and \eqref{eqn:term_chisquared}), the last inequality uses the fact that $\sqrt{a+b} \leq \sqrt{a} + \sqrt{b}$. For every variable $j \not\in S$, we have 
\begin{align*}
\norm{\hat B_j}_2 \leq \frac{91 C_{\mu} \sigma |S| \sqrt{m} \log p }{\kappa n} + 3 C_M \sigma \sqrt{\frac{m + \log p}{n}}.
\end{align*}
plug in $\kappa \geq \frac{1}{2} \lambda_{\min}$, $C_{\mu} = 10 e \sigma_X^4 \sqrt{\frac{\lambda_{\max}}{\lambda_{\min}}}, C_M \leq 2K$, we obtain
\[
\norm{\hat B_j}_2 \leq \frac{1820 e \sigma_X^4 \lambda_{\max}^{1/2}  \sigma |S| \sqrt{m} \log p}{\lambda_{\min}^{3/2} n} + 6 K \sigma \sqrt{\frac{m + \log p}{n}}.
\]
From \eqref{eqn:bound_sc} and the choice of $\Lambda^*$, we see that
all variables not in $S$ will be excluded from $\hat S$ as well. For
every variable $j \in S$, we have
\begin{align*}
\norm{\hat B_j}_2 \geq \norm{ B_j}_2 - \norm{\tilde B_{j} - B_j}_2 \geq 2 \Lambda^* - \Lambda^* = \Lambda^*.
\end{align*}
Therefore, all variables in $S$ will correctly stay in $\hat S$ after the
group hard thresholding.

\section{Proof of Corollary~\label{cor:prediction_estimation}}

From Theorem 2 we have that $\hat S(\Lambda^*) = S$ and 
\begin{align}
\norm{\tilde B_{j} - B_j}_2 
\leq \frac{1820 e \sigma_X^4 \lambda_{\max}^{1/2}  \sigma |S| \sqrt{m} \log p}{\lambda_{\min}^{3/2} n} + 6 K \sigma \sqrt{\frac{m + \log p}{n}},
\label{eqn:row_error}
\end{align}
with high probability. Summing over $j \in S$, we obtain the $\ell_1/\ell_2$ estimation
error bound. For the prediction risk bound, we have
\begin{align*}
\frac{1}{nm} \sum_{t=1}^m \norm{X_t (\tilde \beta_t - \beta^*_t)  }_2^2 
\leq& \frac{\lambda_{\max}}{m}  \sum_{i=1}^m \norm{\tilde \beta_t - \beta^*_t}_2^2 \\
=& \frac{\lambda_{\max}}{m} \sum_{j=1}^p \norm{\tilde B_j - B_j}_2^2.
\end{align*}
Using \eqref{eqn:row_error} and the fact that $\tilde B - B$ is
row-wise $|S|$-sparse, we obtain the prediction risk bound.

\section{Collection of known results}

For completeness, we first give the definition of subgaussian norm, details could be found at \citep{Vershynin2012Introduction}. 

\begin{definition}[Subgaussian norm]
The subgaussian norm $\norm{X}_{\psi_2}$ of a subgaussian $p$-dimensional random vector $X$, is defined as 
\[
\norm{X}_{\psi_2} = \sup_{x \in \mathbb{S}^{p-1}} \sup_{q > 1} q^{-1/2}  (\EE |\dotp{X}{x}|^q)^{1/q},
\]
where $\mathbb{S}^{p-1}$ is the $p$-dimensional unit sphere.
\end{definition}

We then define the restricted set $\Ccal(|S|,3)$ as 
\[
\Ccal(|S|,3) = \{ \Delta \in \RR^p | \norm{\Delta_{U^c}}_1 \leq 3 \norm{\Delta_U}_1, U \subset [p], |U| \leq |S| \}.
\]

The following proposition is a simple extension of Theorem 6.2 in
\citep{Bickel2009Simultaneous}.

\begin{proposition}
Let
\[
\lambda_t = A \sigma \sqrt{\frac{\log p}{n}}
\]
with some constant $A > 2 \sqrt{2}$ be the regularization parameter in
lasso. With probability at least $1 - mp^{1 - A^2/8}$,
\[
\norm{\hat \beta_t - \beta^*_t}_1 \leq \frac{16 A}{\kappa'} \sigma |S| \sqrt{\frac{\log p}{n}},
\]
\end{proposition}
where $\kappa$ is the minimum restricted eigenvalue of design matrix $X_1,\ldots,X_m$:
\[
\kappa = \min_{t \in [m]} \min_{\Delta \in \Ccal(|S|,3)} \frac{\Delta^T \rbr{\frac{X_t^T X_t}{n}} \Delta}{\norm{\Delta_S}_2^2}.
\]
\begin{proof}
Using Theorem 6.2 in
\citep{Bickel2009Simultaneous} and take an union bound over $1,\ldots,m$ we obtain the result.
\end{proof}

\begin{lemma}[Equation (27) in \citep{cavalier2002oracle}; Lemma B.1 in \citep{Lounici2011Oracle}]
\label{lemma:chi_squared}
Let $\xi_1,\xi_2,...\xi_m$ be i.i.d. standard normal random variables,
let $v = (v_1,...,v_m) \neq 0$, $\eta_{v} = \frac{1}{\sqrt{2}
  \norm{v}_2} \sum_{i=1}^m (\xi_i^2 - 1)v_i$ and $m(v) =
\frac{\norm{v}_{\infty}}{\norm{v}_{2}}$. We have, for all $t > 0$,
that
\[
P(|\eta_v| > t) \leq 2 \exp \rbr{ - \frac{t^2}{2 + 2 \sqrt{2} t m(v)} }.
\] 
\end{lemma}

The next lemma relies on the generalized coherence parameter:
\begin{definition}[Generalized Coherence]
For matrices $X\in\RR^{n\times p}$ 
and $M = (m_1,\ldots,m_p)\in\RR^{p\times p}$, let 
\[
\mu(X,M) = \max_{j\in[p]} \norm{\Sigma m_j - e_j }_{\infty}
\]
be the generalized coherence parameter between $X$ and $M$, where
$\Sigma = n^{-1}X^TX$. Furthermore, let 
$\mu^* = \min_{t\in[m]}\min_{M\in\RR^{p\times p}}\mu(X_t,M)$ be the
minimum generalized coherence.
\end{definition}

\begin{lemma}[Theorem 2.4 in \citep{Javanmard2013Confidence}] 
\label{lemma:c_mu}
When $X_t$ are drawn from subgaussian random vectors with covariance matrix $\Sigma_t$, and $X_t \Sigma^{-1/2}_t$ has bounded subgaussian norm $\norm{X_t \Sigma^{-1/2}_t}_{\psi_2} \leq \sigma_X$. When $n \geq 24 \log p$, then with probability at least $1 - 2p^{-2}$, we have
\[
\mu(X_t,\Sigma_t^{-1}) < 10 e \sigma_X^4 \sqrt{\frac{\lambda_{\max}}{\lambda_{\min}}} \sqrt{\frac{\log p}{n}}.
\]
\end{lemma}

For subgaussian design, we also have the following restricted eigenvalue condition \citep{rudelson2011reconstruction,lee2015communication}.
\begin{lemma}
When $X_t$ are drawn from subguassian random vectors with covariance matrix $\Sigma_t$, and bounded subgaussian norm $\sigma_X$. When $n \geq 4000 s' \sigma_X \log \rbr{\frac{60 \sqrt{2} e p}{s'}}$ where $s' = \rbr{1 + 30000 \frac{\lambda_{\max}}{\lambda_{\min}}}|S|$, and $p > s'$, then with probability at least $1 - 2 \exp^{(-n/{4000C_{\kappa}^4})}$, for any vector $\Delta \in \Ccal(|S|,3)$ where
we have
\[
\Delta^T \rbr{\frac{X_t^T X_t}{n}} \Delta \geq \frac{1}{2} \lambda_{\min} \norm{\Delta_S}_2^2.
\]
\end{lemma}

\end{document}